\documentclass[11pt]{article}

\usepackage{url}
\usepackage{fullpage,bm}
\usepackage{color}
\usepackage{amsmath,amssymb,amsthm}
\usepackage{amsmath,amssymb,mathtools}
\usepackage{bm}
\usepackage{amsmath}
\usepackage{amssymb}
\usepackage{latexsym}
\usepackage{verbatim}
\usepackage{subfigure}
\usepackage[final]{graphicx}
\usepackage{psfrag}
\usepackage{xcolor}
\usepackage{lipsum}
\usepackage{authblk}
\usepackage{blindtext}
\usepackage{mathrsfs}  
\usepackage{geometry}
\usepackage{algorithm, algorithmic}
\usepackage{cite}

\newtheorem{theorem}{Theorem}[section]

\newtheorem{lemma}{Lemma}[section]
\newtheorem{corollary}{Corollary}[section]
\newtheorem{remark}{Remark}

\newtheorem{definition}{Definition}[section]
\newtheorem{example}{Example}[section]

\DeclareMathOperator*{\argmax}{arg\,max}
\DeclareMathOperator*{\argmin}{arg\,min}

\def\htheta{\hat{\theta}^\cM_{\min}}
\def\ttheta{\tilde{\theta}^\cM_{\min}}
\def\detheta{\theta^{\cM,DRO}_{\varepsilon,\min}}
\def\hdetheta{\hat{\theta}^{\cM,DRO}_{\varepsilon,\min}}
\def\gtheta{\hat{\theta}^{\cGL}_{\min}}
\def\dhetheta{\Delta\hat{\theta}^{\cM}_{\varepsilon,\min}}
\def\dmetheta{\Delta\theta^{\cM}_{\varepsilon,\min}}
\def\etheta{\theta^{\cM}_{\varepsilon,\min}}
\def\hetheta{\hat{\theta}^{\cM}_{\varepsilon,\min}}

\def\mtheta{\theta^{\cM}_{\min}}
\def\mS{\mathcal{S}_{\varepsilon}(\mathcal{M})}
\def\hmS{\hat{\mathcal{S}}_{\varepsilon}(\mathcal{M})}
\def\ketheta{\hat{\theta}_{\varepsilon,\min}}
\def\ktheta{\hat{\theta}_{\min}}
\def\R{\mathbb{R}}

\def\E{\mathbb{E}}
\def\bE{\mathbb{E}}
\def\bP{\mathbb{P}}
\def\bR{\mathbb{R}}

\def\v{v^{\cGL}}

\def\Pb{{ \mathbb{P} }}

\def\cX{\mathcal{X}}
\def\cY{\mathcal{Y}}
\def\cR{\mathcal{R}}
\def\cL{\mathcal{L}}
\def\cS{\mathcal{S}}
\def\cI{\mathcal{I}}
\def\cM{\mathcal{M}}
\def\cI{\mathcal{I}}
\def\cN{\mathcal{N}}
\def\cK{\mathcal{K}}
\def\cGL{\mathcal{GL}}
\def\cL{\mathcal{L}}
\def\cQ{\mathcal{Q}}
\begin{document}
\thispagestyle{empty}
\title{Interpreting Robust Optimization via Adversarial Influence Functions}

\author{Zhun Deng\thanks{Harvard University, zhundeng@g.harvard.edu}
\qquad Cynthia Dwork\thanks{Harvard University, dwork@seas.harvard.edu}
\qquad Jialiang Wang\thanks{Harvard University, jialiangwang@g.harvard.edu}
\qquad Linjun Zhang\thanks{Rutgers University, lz412@stat.rutgers.edu}
}

\date{}
\maketitle

\abstract
Robust optimization has been widely used in nowadays data science, especially in adversarial training. However, little research has been done to quantify how robust optimization changes the optimizers and the prediction losses comparing to standard training.  In this paper, inspired by the influence function in robust statistics, we introduce the Adversarial Influence Function (AIF) as a tool to investigate the solution produced by robust optimization. The proposed AIF enjoys a closed-form and can be calculated efficiently. To illustrate the usage of AIF, we apply it to study model sensitivity --- a quantity defined to capture the change of prediction losses on the natural data after implementing robust optimization. We use AIF to analyze how model complexity and randomized smoothing affect the model sensitivity with respect to specific models.  We further derive AIF for kernel regressions, with a particular application to neural tangent kernels, and experimentally demonstrate the effectiveness of the proposed AIF. Lastly, the theories of AIF will be extended to distributional robust optimization.

\section{Introduction}
Robust optimization is a classic field of optimization theory that seeks to achieve a certain measure of robustness against uncertainty in the parameters or inputs involved \cite{ben2009robust, beyer2007robust}. Recently, it has been used to address a concern in deep neural networks --- the deep neural networks are vulnerable  to adversarial perturbations \cite{goodfellow2014explaining, szegedy2013intriguing}.

 In supervised learning, given input $x$, output $y$ and a certain loss function $l$,  adversarial training through robust optimization for a model $\cM$ is formulated as 
\begin{equation}\label{eq:model}
\min_{\theta^\cM\in\Theta}\bE_{{x,y}}\max_{\delta\in\cR(x)}l(\theta^{\cM},x+\delta,y,\cM),
\end{equation}
where $\cR(x)$ is some constrained set, which is usually taken as a small neighborhood of $x$ in  robust optimization. For example, in image recognition \cite{he2016deep}, an adversarial attack should be small so that it is visually imperceptible. 

Although adversarial training through robust optimization has achieved great success in defending against adversarial attacks \cite{madry2017towards}, the influence of such adversarial training on predictions is under-explored, even for a simple model $\cM$. In particular, let us define the regular optimizer and the robust optimizer respectively:
\begin{align}
\mtheta &:=\argmin_{\theta^{\cM}\in\Theta}\bE_{{x,y}}l(\theta^{\cM},x,y,\cM),\nonumber\\
\etheta &:=\argmin_{\theta^\cM\in\Theta}\bE_{{x,y}}\max_{\delta\in\cR(x,\varepsilon)}l(\theta^{\cM},x+\delta,y,\cM).\label{eq:optimization}
\end{align}
It is unclear how $\bE_{{x,y}}l(\etheta,x,y,\cM)$ --- the prediction loss on the original data with robust optimizer---  performs compared to the optimal prediction loss $\bE_{{x,y}}l(\mtheta,x,y,\cM)$. The difficulty for studying this questions is the underlying NP-hardness of solving robust optimization. Even for the simple models, say quadratic models, the robust optimization problem is NP-hard if the constraint set is polyhedral \cite{minoux2010robust}.

To address this problem, drawing inspiration from the idea of influence function in robust statistics \cite{hampel1968contributions,hampel1974influence, croux1999influence,huber2009robust}, which characterizes 
how the prediction loss changes when a small fraction of data points being contaminated, we propose the Adversarial Influence Function (AIF) to investigate the influence of robust optimization on prediction loss. Taking advantage of small perturbations, AIF has a closed-form expression and can be calculated efficiently. Moreover,  AIF enables us to analyze the prediction error without implementing the robust optimization, which typically takes long time due to the computational burden of searching adversaries. 

The rest of the paper is organized as follows. Section \ref{sec:setup} lays out the setup and notation. Section \ref{sec:sensitivity} defines model sensitivity, which is used to understand how robust optimization affects the predictions. To efficiently approximate the model sensitivity, Section \ref{sec:aif} introduces the AIF. Further, in Section \ref{sec:usecase}, we show several case studies, by applying the proposed AIF to theoretically analyze the relationship between model sensitivity and model complexity and  randomized smoothing. In Section \ref{sec:extensions}, we extend the AIF theory to kernel regressions and distributional robust optimization. 
\vspace{-0.2cm}
\subsection{Related work}
\noindent\textbf{Adversarial training and robust optimization}\quad Since \cite{goodfellow2014explaining} proposed adversarial training, many innovative methods have been invented to improve the performance of adversarial training, such as \cite{shafahi2019adversarial, agarwal2018learning, liu2019rob, yin2018rademacher}. Earlier work only added adversarial examples in a few rounds during training, and many of them have been evaded by new attacks \cite{athalye2018obfuscated}. In \cite{madry2017towards}, the authors proposed to use projected gradient ascent and obtain the state-of-art result.  They further pointed out that the adversarial training can be formulated through the lens of robust optimization. Nevertheless, robust optimization has a very deep root in engineering \cite{taguchi1989quality} , but many robust optimization problems are NP- hard\cite{minoux2010robust}, and solving such problems heavily relies on high-speed computers and their exponentially increasing FLOPS-rates \cite{park2006robust}. Our adversarial influence function may bridge the gap between theoretical analysis and engineering implementation of robust optimization to a certain degree, and improve our understanding of robust optimization.

\noindent\textbf{Robust Staistics} Robust statistics has been recently applied to machine learning and achieves impressive successes. \cite{koh2017understanding} used the influence function to understand the prediction of a black-box model. \cite{debruyne2008model, liu2014efficient} and \cite{christmann2004robustness} used the influence function for model selections and cross-validations in kernel methods. Recently, \cite{bayaktar2018adversarial} extended the influence function to the adversarial setting, and investigated the adversarial robustness of multivariate M-Estimators. We remark here that their adversarial influence function is different from ours, where they focused on the influence on parameter inference, while ours focus on the influence of robust optimization on the prediction.

\section{Setup and Notation}\label{sec:setup}
In this paper, we consider the task of mapping $m$-dimensional input $x\in\cX\subseteq \bR^m$ to a scalar output $y\in\cY$, with joint distribution $(x,y)\sim\bP_{x,y}$ and marginal distributions $x\sim \bP_{x}$, $y\sim \bP_{y}$ . We have training dataset $(X^t,Y^t)=\{(x^t_1, y^t_1),\cdots,(x^t_{n_t}, y^t_{n_t})\}$ and evaluation dataset $(X^e,Y^e)=\{(x^e_1, y^e_1),\cdots,(x^e_{n_e}, y^e_{n_e})\}$.  For a given model architecture $\cM$, the loss function is denoted as $l(\theta^{\cM},x,y,\cM)$ with parameter $\theta^{\cM}\in \Theta\subseteq \bR^d$ (we will omit $\cM$ in $l$ sometimes if not causing confusions). For robust optimization, we focus on studying the constraint set $\cR(x,\varepsilon)=\{\omega\in\cX:\|\omega-x\|_p\leq\varepsilon\cdot \bE_{x\sim\bP_x}\|x\|_p \}$ \textbf{with small $\varepsilon$}, where $\|\cdot\|_p$ is the $l_p$ norm. Such type of constraint set is also called $l_p$-attack in adversarial learning, which implies the adversaries are allowed to observe the whole dataset and are able to contaminate each data point $x_i$ a little bit. This is commonly used in adversarial training for image classifications in machine learning and the constant factor $ \bE_{x\sim\bP_x}\|x\|_p$  is for scale consideration.\footnote{For standard MNIST, the average  $l_2$ norm of $x$ is  9.21 with dimension $m=28\times 28$. The attack size does not have to be small, but $\varepsilon$, as the ratio of the magnitude of adversarial attacks and average magnitude of images,  is small. }

Further, we denote the empirical version of the minimizers for regular optimization and robust optimizers in Eq. \eqref{eq:optimization}:
\begin{align*}
\htheta &:=\argmin_{\theta^{\cM}\in\Theta}\frac{1}{n_t}\sum_{i=1}^{n_t}l(\theta^{\cM},x^t_i,y^t_i,\cM),\\
\hetheta &:=\argmin_{\theta^\cM\in\Theta}\frac{1}{n_t}\sum_{i=1}^{n_t}\max_{\delta_i\in\hat{\cR}(x^t_i,\varepsilon)}l(\theta^{\cM},x^t_i+\delta_i,y^t_i,\cM),
\end{align*}
where $\hat{\cR}(x^t_i,\varepsilon)=\{u\in\cX:\|u-x^t_i\|_p\leq\varepsilon \hat\bE_{x^t}\|x\|_p \}$, with $\hat\bE_{x^t}$ being the expectation with respect to the empirical probability distribution of $x^t$. 

We use $\text{sgn}(x)$ to denote the sign function: $\text{sgn}(x)=1$ if $x>0$, $\text{sgn}(x)=0$ if $x=0$, and $-1$ otherwise. We also use $[n]$ to denote the set $\{1,2\cdots, n\}$. Further, we use the notion $o_p$ and $O_p$, where for a sequence of random variables $X_n$, $X_n=o_p(a_n)$ means $X_n/a_n\to0$ in probability, and $X_n=O_p(b_n)$ means that for any $\varepsilon>0$, there is a constant $K$, such that $\mathbb P(|X_n|\le K\cdot b_n)\ge 1-\varepsilon$.

 \section{Model Sensitivity}\label{sec:sensitivity}
In order to quantify how robust optimization affects predictions, we first define the model sensitivity with respect to the robust optimization. 
\begin{definition}[$\varepsilon$-sensitivity/adversarial cost]\label{def:epsilonsensitivity}
For a given model $\cM$, the $\varepsilon$-sensitivity/adversarial cost is defined as  
\begin{equation*}
\mS:=\bE_{{x,y}}l(\etheta,x,y,\cM)-\bE_{{x,y}}l(\theta^{\cM}_{\min},x,y,\cM).
\end{equation*}
\end{definition}
 The $\varepsilon$-sensitivity/adversarial cost quantifies how robust optimization increases the expected loss, and this loss also indicates the additional cost of \textit{being adversarially robust}.  Besides this straightforward interpretation, one can also interpret $\mS$ as a trade-off between the prediction loss and robustness for model architecture $\cM$ ---  the optimizer $\etheta$ is more adversarially robust but inflates the prediction loss comparing to $\theta^{\cM}_{\min}$. For fixed $\varepsilon$,  an architecture $\cM$ with small $\varepsilon$-sensitivity implies that such an architecture can  achieve adversarial robustness by robust optimization 
 without sacrificing the performance on the original data too much. We also say an architecture $\cM$ with smaller $\varepsilon$-sensitivity is \emph{more stable}.
  
Since $\theta^{\cM}_{\min}$ is the minimizer of $\bE_{{x,y}}l(\theta^{\cM},x,y,\cM)$ over $\theta^{\cM}$, if we further have $\theta^{\cM}_{\min}\in \Theta^\circ$, where $\Theta^\circ$ denotes the interior of $\Theta$ and $l$ is twice differentiable, by Taylor expansion, we would have 
\begin{align*}
\mS =&\frac{1}{2}(\dmetheta)^T\bE_{{x,y}}\nabla^2l(\theta^{\cM}_{\min},x,y,\cM)\dmetheta \\&+o(\|\dmetheta\|_2^2),
\end{align*}
where $\dmetheta=\etheta-\mtheta$, and the remainder is negligible if $\varepsilon$ is small enough. 
 Given the training set $(X^t,Y^t)$ and the evaluation set $(X^e,Y^e)$, we define the empirical $\varepsilon$-sensitivity:
\begin{equation}\label{eq:approximationofsenstivity}
\hmS\approx\frac{1}{2}(\dhetheta)^T \bE_{ \hat{\bP}_{x^e,y^e}}\nabla^2l(\htheta,x,y,\cM)\dhetheta,
\end{equation}
by omitting the remainder $o(\|\dhetheta\|_2^2)$, where $\dhetheta = \hetheta-\htheta$. Notice that Eq. \eqref{eq:approximationofsenstivity}  involves $\dhetheta$, the solution of  robust optimization, which, even for  simple models with loss functions (such as linear regression with quadratic loss), does not have a closed-form expression and is computationally heavy to obtain. In the following sections, we will address this problem by introducing AIF, which provides an efficient way to approximate and analyze $\mS$. For simplicity of illustration, we remove the superscripts $t, e$ and use generic notation $(X,Y)=\{(x_1, y_1),\cdots,(x_{n}, y_{n})\}$ for general dataset in the following sections when there is no ambiguity.

 \section{Adversarial Influence Function}\label{sec:aif}
Unless explicitly stated, we mainly consider the case where the empirical risk $\sum_{i=1}^nl(\theta^{\cM},x^t_i,y^t_i;\cM)$ \textbf{is twice differentiable and strongly convex} in this paper. A relaxation of such conditions will be discussed in Section~\ref{sec:relaxation}. In order to approximate $\hetheta-\htheta$, for small $\varepsilon$, we use 
 \fontsize{9.3pt}{9.3}\selectfont
$$\hetheta-\htheta\approx\varepsilon^\alpha\cdot \frac{d(\hetheta-\htheta)}{d\varepsilon^\alpha}|_{\varepsilon=0+}=\varepsilon^\alpha\cdot\frac{d\hetheta}{d\varepsilon^\alpha}\big|_{\varepsilon=0+}
$$ 
\normalsize
for approximation, 
where $\alpha>0$ is the smallest positive real number such that the limit $\lim_{\varepsilon\to0+}{(\hetheta-\htheta)}/{\varepsilon^\alpha}$ is nonzero. Throughout this section, all the cases we consider later have $\alpha=1$, while more general cases will be discussed in Section~\ref{subsec:daif}. Formally, we define the adversarial influence function as  follows. 
\begin{definition}[Adversarial Influence Function]\label{def:epsilonsensitivity}
For a given model $\cM$, the adversarial influence function (AIF) is defined as  
\begin{equation}\label{eq:aifdef}
\cI(\cM):=\frac{d\etheta}{d\varepsilon}\big|_{\varepsilon=0}.
\end{equation}
\end{definition}
The AIF measures the changing trend of the optimizer  under robust optimization in the limiting sense. 
With the help of AIF, we then approximate $\cS_\varepsilon(\cM)$ by 
$$\cS_\varepsilon(\cM)\approx \frac{1}{2}\varepsilon^2\cI(\cM)^T\bE_{{x,y}}\nabla^2l(\theta^{\cM}_{\min},x,y,\cM)\cI(\cM)\big|_{\varepsilon=0}$$
when $\varepsilon$ is small.

Next we provide a specific characterization of the empirical adversarial influence functions. We denote $\hat{I}(\cM)=d\hetheta/d\varepsilon|_{\varepsilon=0}$ as the empirical version of AIF. Besides, we denote the perturbation vector as $\Delta=(\delta_1^T,\cdots,\delta_n^T)^T$. Further, for given $(X,Y)$ and $\cM$, we define  $g(\theta^\cM,\Delta)=1/n\sum_{i=1}^{n}l(\theta^{\cM},x_i+\delta_i,y_i;\cM)$ when we only consider the  optimization over $(\theta^\cM,\Delta)$. 
\begin{theorem}\label{thm:firstorder}
Suppose $\mathcal{X}$, $\mathcal{Y}$ and $\Theta$ are compact spaces, the loss function $l(\theta, x, y)$ is three times continuously differentiable on $(\theta, x)\in\Theta\times \mathcal{X}$ for any given $y\in\mathcal Y$, and the empirical Hessian matrix $\hat{H}_{\htheta}=1/n \sum_{i=1}^n\nabla^2_\theta l(\htheta,x_i,y_i)$ is positive definite. Further, we assume the empirical risk $\sum_{i=1}^nl(\theta^{\cM},x^t_i,y^t_i;\cM)$ is twice differentiable and strongly convex and $g(\cdot,\Delta)$ is differentiable for every $\Delta$, $\nabla_\theta g(\theta^\cM,\Delta)$ is continuous on $\Theta\times \mathcal{X}$, $\htheta$ lies in the interior of $\Theta$, and $\nabla_x l(\htheta, x_i,y_i,\cM)\neq 0$ for all $i\in[n]$, then we have
\begin{equation}\label{eq:aif}
\hat{\cI}(\cM)=-\hat{H}_{\htheta}^{-1}\Phi, 
\end{equation}
where $\Phi=1/n\sum_{i=1}^n\nabla_{x,\theta} l(\htheta,x_i,y_i)\bE_{x\sim\hat{\bP}_{x}}\|x\|_p\phi_i$ and $\phi_i=(\psi_1,\psi_2,\cdots,\psi_m)^T$, with 
$$\psi_k=\frac{b_k^{q-1}}{(\sum_{k=1}^m b_k^q)^\frac{1}{p}}\text{sgn}\Big(\frac{\partial}{\partial x_{\cdot,k}}l(\htheta,x_i,y_i,\cM)\Big).$$
Here, we have $b_k=\big|\frac{\partial}{\partial x_{\cdot,k}}l(\htheta,x_i,y_i,\cM)\big|$, $x_{\cdot,k}$ is the k-th coordinate of vector $x$, for instance, $x_j=(x_{j,1},x_{j,2},\cdots,x_{j,m})^T$; $p\ge0$ and $q\ge0$ are conjugate such that $1/p+1/q=1$.
\end{theorem}
\begin{remark}
The  compactness condition is easy to satisfy. Since for any distributions $\mathcal{D}$ and integer $n$, we can take a sufficiently large constant $R>0$, which is allowed to depend on $n$, such that all $n$ samples are contained in the ball $\mathbb{B}(0,R)$ with high probability. Besides, if the input $x$ is of high dimension, the computational bottleneck is mainly on inverting the empirical Hessian. We can use techniques such as conjugate gradients and stochastic estimation suggested in \cite{koh2017understanding} to reduce the computational cost. 
\end{remark}
{The above theorem provides a closed-form expression for the first order AIF, and therefore a closed-form approximation of the model sensitivity $\mS$. One nice property of such an approximation is that it does not depend on optimization algorithms, but only depends on the model $\cM$ and the distribution of $(x, y)$. This attribute makes model sensitivity an inherent property of model $\cM$ and data distribution, making it a potential new rule for model selection. Model sensitivity can help us pick those models whose prediction result will not be greatly affected after robust optimization.  }

We show the effectiveness of approximation by AIF in Figure \ref{fig:effectivelinear}. We plot two error curves for $\Delta\hat{I}(n,\varepsilon):=\|(\hetheta-\htheta)/\varepsilon-\hat{\cI}(\cM)\|_2$ and $\Delta\hat{S}(n,\varepsilon):=\|\hmS/\varepsilon^2-(\hat{\cI}(\cM))^T\bE_{ \hat{\bP}_{x^e,y^e}}\nabla^2l(\htheta,x,y,\cM)\hat{\cI}(\cM)\|_2$, where the sample size is $n$. Theoretically, we expect $\Delta\hat{I}(n,\varepsilon)$ and  $\Delta\hat{S}(n,\varepsilon)$ go to $0$ as $\varepsilon$ goes to $0$. \textbf{In all the experiments in the paper, we use projected gradient descent (PGD) for robust optimization to obtain $\hetheta$}. In Figure 1, we can see that as $\varepsilon$ become smaller, $\Delta\hat{I}(n,\varepsilon)$ and   $\Delta\hat{S}(n,\varepsilon)$ gradually go to $0$. We remark here that we do not let $\varepsilon$ be exactly $0$ in our experiments, since PGD cannot obtain the exact optimal solutions for $\htheta$ and $\hetheta$. The existing system error will become dominating if $\varepsilon$ is too small and return abnormally large value after divided by $\varepsilon$. This also motivates us to introduce the AIF to have an accurate approximation. The model we use is a linear regression model with $500$ inputs drawn from a two-dimensional standard Gaussian, i.e.  $x\sim\cN(0,I)$. We fit $y$ with $y=2x_1-3.4x_2+\eta$ and $\eta\sim 0.1\cdot\cN(0,I)$.

\begin{figure}[ht]
\vskip -0.1in
\begin{center}
\centerline{\includegraphics[width=2.8 in]{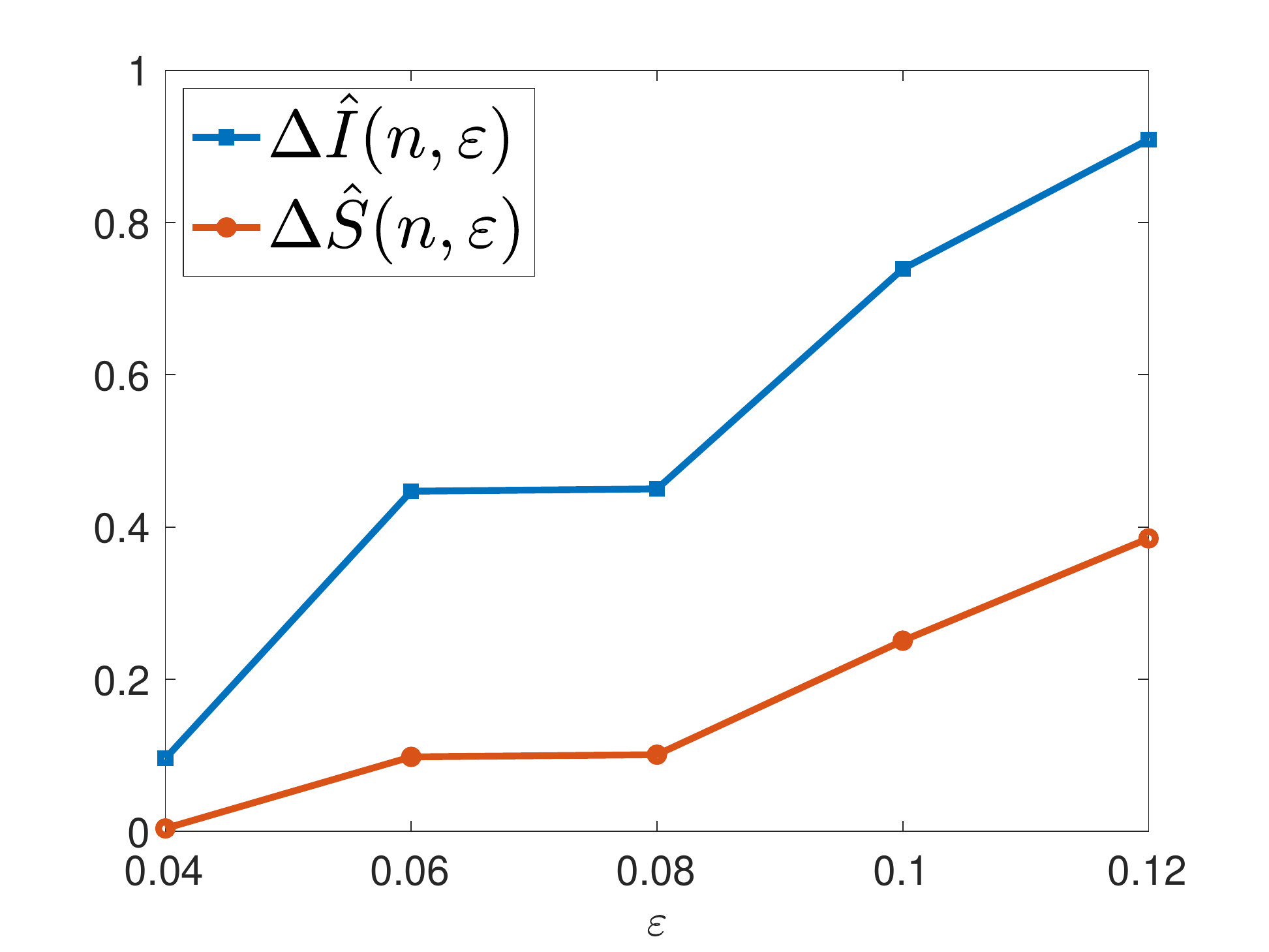}}
\caption{Effectiveness of AIF and model sensitivity for linear regression model. From the monotonicity relationship between $\varepsilon$ and $\Delta\hat{I}(n,\varepsilon)$, $\Delta\hat{S}(n,\varepsilon)$, we verify the effectiveness of AIF and model sensitivity. Here, the sample size $n=500$.}\label{fig:effectivelinear}
\end{center}
\vskip -0.2in
\end{figure}
\begin{remark}
It is straightforward to derive asymptotic normality for AIF by  central limit theorem\cite{durrett2019probability}, which can be used to construct confidence intervals for $\cI(\cM)$. Specifically, if we denote 
$\zeta_i:=-H_{\htheta}^{-1}\nabla_{\theta,x} l(\htheta,x_i,y_i)\bE_{x\sim\hat{\bP}_{X}}\|x\|_p\phi_i$, $\hat{\mu}_n:=1/n\sum_{i=1}^n\zeta_i$, and $\hat{\Sigma}_n:=1/n\sum_{i=1}^n(\zeta_i-\hat{\mu}_n)(\zeta_i-\hat{\mu}_n)^T,$
then by classic statistical theory, we obtain
\begin{equation*}
\sqrt{n}\hat{\Sigma}^{-1/2}_n(\hat{\cI}(\cM)-\hat{\mu}_n)\stackrel{\mathcal D}{\rightarrow} \cN(0,I_d),
\end{equation*}
as $n$ goes to infinity, where  $\cN(0,I)$ denotes  standard multivariate normal distribution and $\stackrel{\mathcal D}{\rightarrow}$ denotes convergence in distribution. 
\end{remark}


\subsection{Non-convex, non-convergence cases}\label{sec:relaxation}
In the previous discussions, we talked about the case where the empirical loss is strongly convex. Now we briefly discuss about non-convex and non-convergence cases.

\noindent \textbf{Well-separated condition.}
In the proof of  Theorem \ref{thm:firstorder}, actually we only need $\htheta$ to be the global minimum and at the point $\htheta$, the empirical Hessian matrix is positive definite and the landscape are allowed to have many local minimums. The uniqueness assumption can also be formulated in a more elementary way: if we assume the smoothness of loss function $l$ over $\mathcal{X}\times\Theta$, compactness of $\Theta$ and we only have one global minimum for $\bE_{(x,y)\sim \bP_{x,y}}l(\theta^{\cM},x,y,\cM)$ which lies in the interior of $\Theta$, with positive definite Hessian matrix, and it is \textit{well-separated}, which means that $\forall \omega>0$, there exists $\kappa>0$, such that $\forall \theta^{\cM}$ , if $\|\theta^\cM-\mtheta\|>\omega,$
we have 
$$|\bE_{{x,y}}l(\theta^{\cM},x,y,\cM)-\bE_{{x,y}}l(\mtheta,x,y,\cM)|> \kappa.$$
By classic statistical theory, $\htheta$ will be a global minimum if sample size is large enough. 

{The well-separated condition relaxes the convexity condition in Theorem \ref{thm:firstorder}. However, the validity of Theorem \ref{thm:firstorder} still requires the condition that $\htheta$ is the global minimum  of the empirical risk, which in practice, is hard to find. Another alternative relaxation is to use a surrogate loss. }

\noindent \textbf{Surrogate losses.} In practice, we may obtain $\ttheta$ by running SGD with early stopping or on non-convex objectives, and get a solution $\htheta$ which is different from $\ttheta$. As in \cite{koh2017understanding}, we can form a convex quadratic approximation of the loss around $\ttheta$, i.e.,
\begin{align*}
\tilde{l}(\theta^\cM,x,y)=&l(\ttheta,x,y)+\nabla_{\theta}l(\ttheta,x,y)(\theta^\cM-\ttheta)\\
&+\frac{1}{2}(\theta^\cM-\ttheta)^T\Big(\nabla^2_{\theta}l(\ttheta,x,y)+\lambda I\Big)(\theta^\cM-\ttheta),
\end{align*}
where $\lambda$ is a damping term to remove the negative eigenvalues of the Hessian. One can show the results of Theorem \ref{thm:firstorder} hold with this surrogate loss.

\section{Case studies of Adversarial Influence Functions}\label{sec:usecase}
To illustrate the usage of adversarial influence functions, we use it to explore the relationship between model complexity, randomized smoothing and model sensitivity. 

\subsection{Model Complexity and Model Sensitivity}
Throughout this paper, we use the term ``model complexity'' as a general term referring to 1) the number of features included in the predictive model, and 2) the model capacity, such as whether the model being linear, non-linear, and so on. 

As observed in the prior literature \cite{madry2017towards, fawzi2018analysis, 45816}, model complexity is closely related to adversarial robustness, that is, when the model capacity increases, the $\varepsilon$-sensitivity/adversarial cost will increase first and then decrease. However, such a phenomenon is only emporical and lack of theoretical justification. 
 In this subsection, we will theoretically explore how the model complexity model affect the model sensitivity/adversarial cost by studying specific models with different model capacity and different number of features included in the predictive model.




\subsubsection{Model Capacity and Model Sensitivity}
 We start with the relationship between model capacity and model sensitivity via two simple and commonly used models, with the dimension of inputs being fixed. 

\paragraph{Linear regression models ($\cL$) and quadratic models ($\cQ$)} We consider the class of linear models $\cL=\{f_\beta(x)=\beta^T x: x,\beta\in\R^m\}$ and the class of quadratic models $\cQ=\{f_{\beta,A}(x)=\beta^T x+x^T A x, x, \beta\in\R^m, A\in \R^{m\times m}\}$.

Apparently, the class of quadratic models has a larger model capacity and  is more flexible than that of linear models. {In the following theorem, we will show that  larger model capacity does not necessarily lead to smaller sensitivity.}
\begin{theorem}\label{linearandquadratic}
We fit the data $(x_i, y_i)$ by $\cL$ and $\cQ$.  For the simplicity of presentation, assume the sample sizes of both the training and testing sample are $n$. Suppose the underlying true generating process is $y= x^T \beta_1^*+( \beta_2^{*T} x)^2+\xi$, where $ x\sim \cN(0,\sigma_x^ 2 I_m)\in\mathbb{R}^m$, $\xi\sim \cN(0,\sigma_\xi^2)$ and independent with $x$.  For $l_2$ or $l_\infty$ attack,
\begin{itemize}
\item[I.] when$
(\|\beta^*_2\|_2^2\sigma_x^2-\sqrt{\frac{2}{\pi}}\sigma_{\xi})^2> \frac{1+2m\sigma_x^2}{\max\{\sigma_x^2, 1\}}\cdot\frac{2}{\pi}\sigma_{\xi}^2,
$
 we have 
$$
\hat{\cS}_{\varepsilon}(\cL)> \hat{\cS}_{\varepsilon}(\cQ)+O_p(\varepsilon^2\sqrt\frac{m^2}{{n}});
$$
\item[II.] when $(\|\beta^*_2\|_2^2\sigma_x^2+\sqrt{\frac{2}{\pi}}\sigma_{\xi})^2<  \frac{1}{\min\{1, \frac{3}{4}\sigma_x^2\}}(1+m\sigma_x^2-2\sigma_x^2\cdot\log m)\cdot\frac{3}{2\pi}\sigma_\varepsilon^2$, then 
$$
\hat{\cS}_{\varepsilon}(\cL)< \hat{\cS}_{\varepsilon}(\cQ)+O_p(\varepsilon^2\sqrt\frac{m^2}{{n}}).
$$
\end{itemize}
\end{theorem}
\vspace{-0.2cm}
From Theorem \ref{linearandquadratic}, unlike adversarial robustness, we can see that the model sensitivity does not have monotonic relationship with the model capacity. Such a monotonic relationship only holds when the model has high complexity (when $\|\beta_2^*\|$ is large). 
Therefore, when $n$ is sufficiently large, the result implies that  a larger model capacity does not necessarily lead to a model with smaller sensitivity.

\subsubsection{Number of features and model sensitivity}
Another important aspect of model complexity is the number of features included in the predictive model. There have been many model selection techniques, such as LASSO, AIC and BIC, developed over years. Given the newly introduced concept of model sensitivity, it is interesting to take model sensitivity into consideration during model selection. For example, if for a specific model, including more features results in a smaller model sensitivity, then for the sake of adversarial robustness
, we should include more features even if it leads to feature redundancy.

For instance, the following results study when $x_i$ follows some structures such as $Cov(x_i)=\sigma_x^2I_m$ for some constant $\sigma_x$, the relationship between model sensitivity and number of features included in linear models.
\begin{theorem}\label{thm:featurelinear}
Suppose that the data $(x_i, y_i)$'s are i.i.d. samples drawn from a joint distribution $P_{x,y}$.  Denote the sample sizes of the training and testing sample by $n_{train}$ and $n_{test}$ respectively.  Let  $m$ be the dimension of input $x$, and
$$\beta^{\cL}_{\min}=\argmin_{\beta}\bE_{P_{x,y}}(y-\beta^Tx)^2.$$
Define  $\eta^{\cL}_i=y_i-\beta^{\cL\top}_{\min}x_i$, and assume $\E[x_i\cdot{\rm{sgn}}(\eta_i^{\cL})]=0$ and $Cov(x_i)=\sigma_x^2I_m$, then for $\ell_2$ attack
$$
\hat{S}_\varepsilon(\cL)=\varepsilon^2(\E_{x\sim\hat P_x}\|x\|_2)^2\cdot(\E|\eta_{i}^{\mathcal L}|)^2\cdot \sigma_x^{-2} +O_p(\varepsilon^2\cdot\sqrt{\frac{1}{n_{train}}+{\frac{m^2}{n_{test}}}}).
$$
\end{theorem}

Given this theorem, we now consider a specific case where we apply this result to random effect model.
\begin{corollary}\label{col:randomeffect}
Consider the random effect model $y=\beta^\top x+\xi$, where $x\in\R^M$, $\beta_1, ..., \beta_M\stackrel{i.i.d.}{\sim} \cN(0,1)$, and $\xi\sim \cN(0,\sigma_\xi^2)$. Further, we assume $x$ is a random design with distribution $x_1,...,x_n\stackrel{i.i.d.}{\sim} \cN(0,\sigma_x^2 I_M)$. Then when we only include $m$ features in the linear predictive model, the resulting model sensitivity is   
\begin{equation}
\hat{S}_\varepsilon(\cL)=\frac{4\varepsilon^2}{\pi\sigma_x^{2}}\frac{\Gamma^2(\frac{m+1}{2})}{\Gamma^2(\frac{m}{2})}\cdot((M-m)\sigma_x^2+\sigma_{\xi}^2) +O_p(\varepsilon^2\cdot\sqrt{\frac{1}{n_{train}}+{\frac{m^2}{n_{test}}}}).
\end{equation}
where $\Gamma(\cdot)$ is the Gamma function such that $\Gamma(x)=\int_0^\infty t^{x-1} e^{-t}\;dt.$
\end{corollary}

Since $\frac{\Gamma^2(\frac{m+1}{2})}{\Gamma^2(\frac{m}{2})}\asymp \frac{1}{2}m$, there is a universal constant $C$, such that $\hat{S}_\varepsilon(\cL)\asymp Cm((M-m)\sigma_x^2+\sigma_{\xi}^2)=-C\sigma_x^2 m^2+C(M+\sigma_\xi^2)m$. This also implies that a larger model capacity does not necessarily lead to a model with smaller sensitivity. Specifically, when $m$ is small, including more features in the linear model results in larger model sensitivity. In contrast, when $m$ is large, i.e. in the high-complexity regime,  including more features leads to smaller model sensitivity.

Next, we consider a broader class of functions --- general regression models.
\paragraph{General regression models ($\cGL$)} In general regression models, suppose we use  a $d$-dimensional basis $\v(x)=(\v_1(x),...,\v_d(x))^T\in\R^d$ to approximate $y$ ($d$ can be a function of $m$), and get the coefficients by solving
 $$\hat{\theta}^{\cGL}_{\min}=\argmin_{\theta} \frac{1}{2n}\sum_{i=1}^n(y_i-\theta^T  \v(x_i))^2,$$
where the loss function is $l(\theta, x_i, y_i,\cGL)=\frac{1}{2}(y_i-\theta^T  \v(x_i))^2$. By Theorem \ref{thm:firstorder}, it is straightforward to obtain
$$\hat{\cI}(\cGL)=-\hat{H}_{\hat{\theta}^{\cGL}_{\min}}^{-1}\Phi=-Cov(\v(x))^{-1}\Phi+O_P(\sqrt\frac{d}{n}),$$
where $Cov(\v(x))$ is the covariance matrix of $\v(x)$ and 
\begin{align*}
&\Phi=\sum_{i=1}^n\Big[\frac{|(\gtheta)^T\v( x_i)-y_i|}{n\|(\gtheta)^T\frac{\partial \v(x_i)}{\partial x}\|} \frac{\partial \v(x_i)}{\partial x} (\frac{\partial \v(x_i)}{\partial x} )^T \gtheta\\
&+ \frac{\v(x_i)}{n}\|(\gtheta)^T\frac{\partial \v(x_i)}{\partial x} \|\text{sgn}((\gtheta)^T \v(x_i)-y_i)\Big].
\end{align*}
Thus, 
\begin{equation}\label{eq:generalregression}
\hat{\cS}_\varepsilon(\cGL)=\varepsilon^2\cdot\Phi^\top Cov(v(x))^{-1}\Phi+O_P(\varepsilon^2\sqrt\frac{d}{n}).
\end{equation}
Notice that the linear regression model is a special case if we take $v(x)=x$. However, the expression of model sensitivity for the general regression models is very complex and hard to analyze directly most of the time. Instead of directly studying Eq. (\ref{eq:generalregression}), we further simplify the expression by providing an upper bound to shed some light.
\begin{theorem}\label{thm:upperbound}
Suppose that the data $(x_i, y_i)$'s are i.i.d. samples drawn from a joint distribution $P_{x,y}$. Let  $m$ be the dimension of input $x$, and
$$\theta^{\cGL}_{\min}=\argmin_{\theta}\bE_{P_{x,y}}(y-\theta^T  \v(x_i))^2.$$
Let  $\eta^{\cGL}_i=y_i-(\theta^{\cGL}_{\min})^T\v(x_i)$ and assume $\E[x_i\cdot{\rm{sgn}}(\eta_i^{\cGL})]=0$, then 
\begin{align*}
\hat{\cS}_\varepsilon(\cGL)\le & \varepsilon^2(\E_{x\sim\hat P_x}\|x\|_2)^2\cdot \frac{1}{\lambda_{\min}(E[v(x_i) v(x_i)^\top])}\cdot\E[\big\|(\frac{\partial}{\partial x} \v(x_i))^T\frac{\partial}{\partial x} \v(x_i)\big\|_2] \cdot \E[|\eta^{\cGL}_i|]^2 \\
&+O_p(\varepsilon^2\sqrt\frac{d}{n}).
\end{align*}
\end{theorem}
The following example illustrates how our upper bound is used to demonstrate the trend of change between model sensitivity and number of features included.
\begin{example}\label{ex}
Suppose $v(x)=(x^T, (\frac{x}{2}\odot \frac{x}{2})^T)^T$. If $x$ consists of random features, such that each coordinate of $x$ is i.i.d drawn from uniform distribution on $(-1, 1)$.  $y= x^T \beta_1^*+ \beta_2^{*T}  \frac{x}{2}\odot \frac{x}{2}+\xi$, where $\xi\sim \cN(0,\sigma_\xi^2)$ and independent with $x$.  As a result, the  eigenvalue satisfies $$
\lambda_{\min}\E[\v(x_i)\v(x_i)^\top]\ge\frac{1}{5};
$$
$$\E[\big\|(\frac{\partial}{\partial x} \v(x_i))^T\frac{\partial}{\partial x} \v(x_i)\big\|_2]=1,$$ 
regardless of the number of features $m$. Besides, $\E|\eta_i^{\cGL}|$ decreases as $m$ increases, and thus the upper bound will decrease as $m$ increases. 

In the experiments in Figure \ref{fig:gapvsbound}, we show the trend for $\hat{\cS}_\varepsilon(\cGL)$ by taking sample size $n=5000$, $\sigma_\xi=0.1$. We take the average result for $1000$ repetitions.

\begin{figure*}[t]
		\centering
		\subfigure[Illustration of the relationship between the feature number and model sensitivity for the model in Example \ref{ex}.   ]{
			\centering
			\includegraphics[scale=0.35]{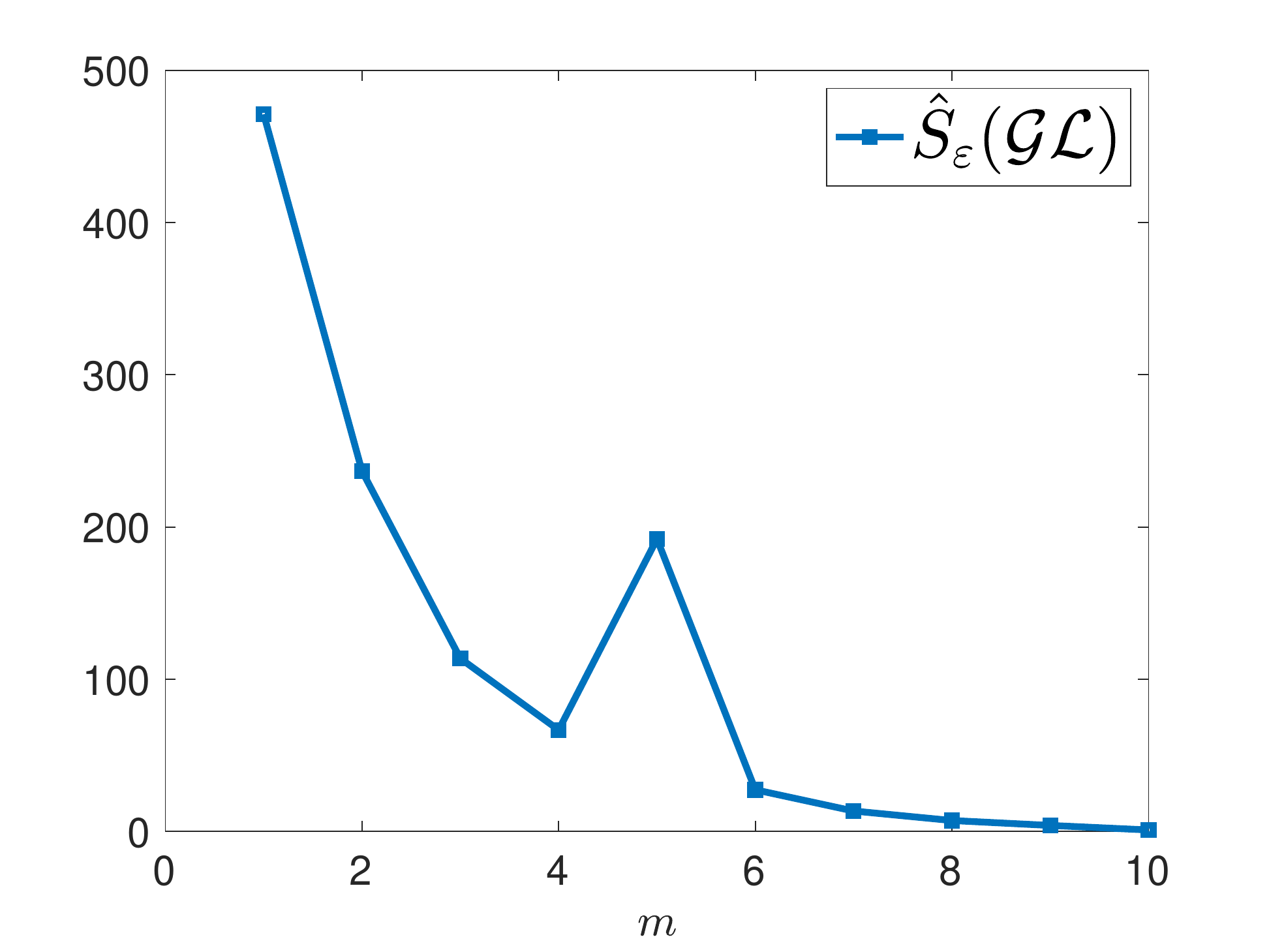}
			\label{fig:gapvsbound}}
        \hspace{0.1in}
        \subfigure[Effectiveness of AIF for kernel regression with NTK on MNIST.]{
		\centering
		\includegraphics[scale=0.35]{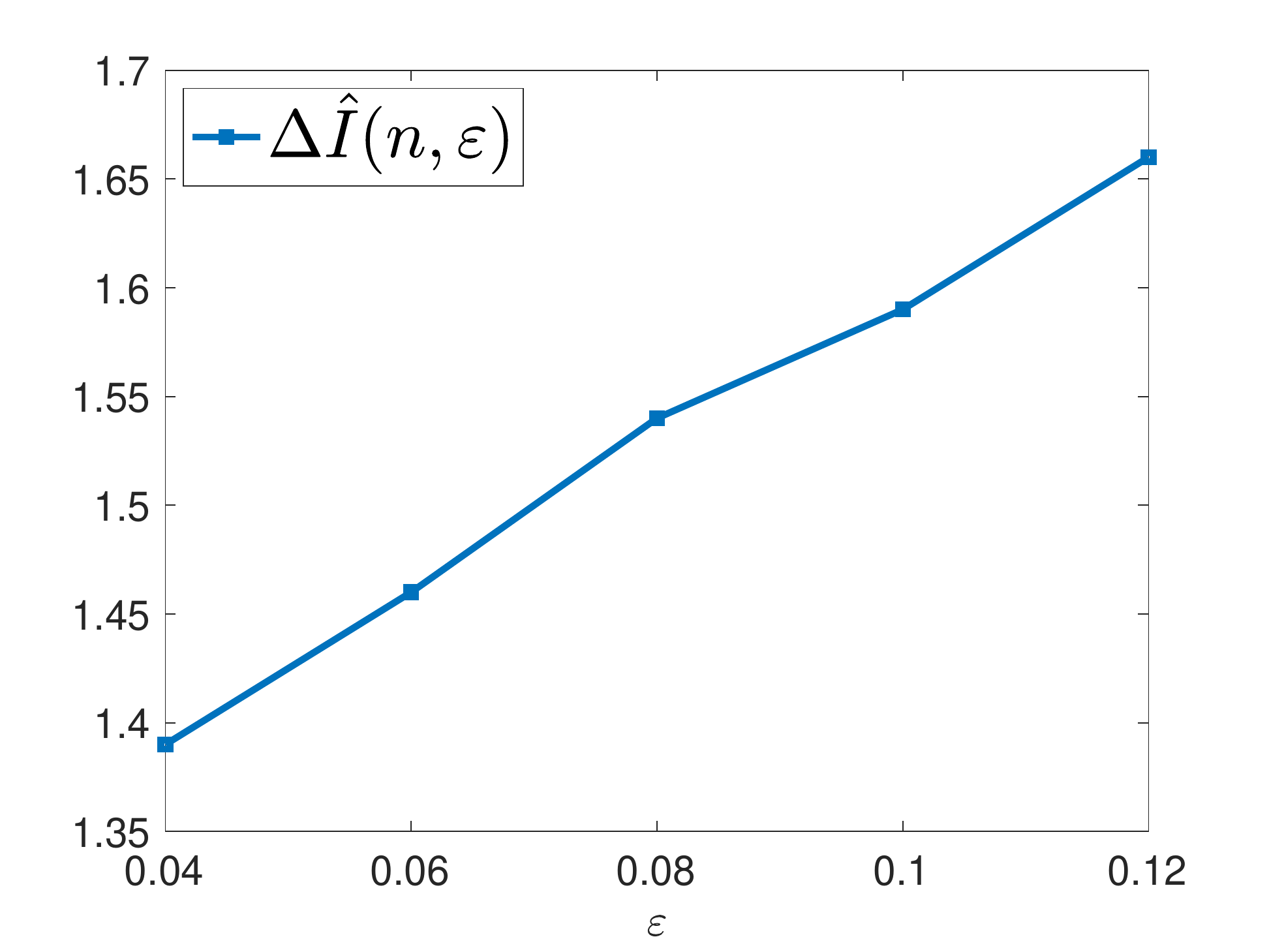}
		\label{fig:mnist}}
        \hspace{0.1in}
 
    \caption{a) Experimentally, the general trend for $\hat{\cS}_\varepsilon(\cGL)$ with respect to $m$ is decreasing (though not strict for every $m$) as the upper bound suggests. b)  The monotonic trend of $\varepsilon$ is still clearly observed, though thevalues are larger than the previous example in Figure \ref{fig:effectivelinear} due to the high dimensionality of MNIST. }
		\label{fig:summary}
\end{figure*}

\end{example}

 \subsection{Randomized Smoothing and Model Sensitivity}\label{subsec:noise}
As the last case study of AIF, we investigate the effect of randomized smoothing \cite{cohen2019certified}, a technique inspired by differential privacy, in adversarial robustness. Randomized smoothing has achieved impressive empirical success as a defense mechanism of adversarial attacks for $l_2$ attack. The core techniques is adding isotropic noise $\vartheta \sim \cN(0,\sigma_r^2 I)$ to the inputs so that for any output range $O$, 
$$\bP\Big(\frac{1}{n}\sum_{i=1}^{n}l(\theta^{\cM},x_i+\vartheta_i,y_i,\cM)\in O\Big)$$
is close to 
$$\bP\Big(\frac{1}{n}\sum_{i=1}^{n}l(\theta^{\cM},x_i+\delta_i+\vartheta_i,y_i,\cM)\in O\Big)$$
for constrained $\|\delta_i\|_2$. 

The following theorem provides an insight into how randomized smoothing affects model sensitivity regarding linear regression models. 
\begin{theorem}\label{thm:rs}
Use the same notation as that in Theorem \ref{thm:featurelinear}. Suppose that the data $(x_i, y_i)$'s are i.i.d. samples drawn from a joint distribution $P_{x,y}$, and  $\E[x_i\cdot{\rm{sgn}}(\eta_i^\cL)]=0$, $Cov(x_i)=\sigma_x^2 I_m$, and $Var(\eta_i^\cL)=\sigma_{\eta^2}$. When we fit $y$ with $\tilde x= x+\vartheta$, where $\vartheta$ is distributed as $N(0,\sigma_r^2 I_m)$, then
$$
\frac{\hat{\mathcal{S}}_\varepsilon(\cL_{\text{noise}})}{\hat{\mathcal{S}}_\varepsilon(\cL)}=\frac{\sigma_x^2/\sigma_{\xi}^2}{\sigma_x^2+\sigma_r^2}\left({\frac{2\sigma_r^2\sigma_x^2}{\sigma_x^2+\sigma_r^2}\|\beta^{\cL}_{\min}\|_2^2+\sigma_{\xi}^2}\right)+O_p(\sqrt{\frac{m}{n}}).
$$
\end{theorem}
Here, $\cL_{\text{noise}}$ denotes the linear model with randomized smoothing by adding input noise. This theorem informs us that when $\sigma_r$ is large, we have $\hat{\mathcal{S}}_\varepsilon(\cL_{\text{noise}})\leq \hat{\mathcal{S}}_\varepsilon(\cL)$ asymptotically, and $\hat{\mathcal{S}}_\varepsilon(\cL_{\text{noise}})$ becomes smaller with larger $\sigma_r$. In other words, the randomized smoothing helps reduce the sensitivity in this case.

\section{Further Extensions}\label{sec:extensions}
In this section, we extend the theories of IFA to kernel regressions and distributional robust optimization. First, we derive the AIF for kernel regressions in Section \ref{subsec:kernel}. In particular, we are interested in how well AIF characterizes the change of optimizers with neural tangent kernels (NTK), whose equivalence to infinitely wide neural networks has been well-established in recent literatures \cite{du2018gradient,jacot2018neural}. In Section \ref{subsec:daif}, we further extend our theory to compute the AIF for distributional robust optimization.
\subsection{AIF of the kernel regressions  }\label{subsec:kernel}
We consider the kernel regression model in the following form
\begin{equation}\label{eq:kernel}
\hat{L}_n(\theta,X,Y):=\frac{1}{n}\sum_{i=1}^n\big(y_i-\sum_{j=1}^nK(x_i,x_j)\theta_j\big)^2+\lambda\|\theta\|^2_2.
\end{equation}
where $\theta=(\theta_1,\cdots,\theta_n)^T$, and $\lambda>0$. Now let us denote $g(\theta,\Delta)= \hat{L}_n(\theta,X+\Delta,Y)$, and we will calculate the empirical adversarial influence function $\hat{\cI}(\cK)$ for kernel $K$.

Notice that in kernel regression, the loss function $\big(y_i-\sum_{j=1}^nK(x_i,x_j)\theta_j\big)^2$ includes all the data points in one single term, which is different from the summation-form of loss function in Theorem \ref{thm:firstorder}. Fortunately, the technique of proving Theorem \ref{thm:firstorder} can still be used here with slight modification. We obtain the following corollary for the adversarial influence function $\hat{\cI}(\cK)$ in kernel regression.

\begin{corollary}\label{col:kernel}
Suppose $\mathcal{X}$, $\mathcal{Y}$ and $\Theta$ are compact spaces, the kernel $\hat L_n$ is three times continuously differentiable on $\Theta\times \mathcal{X}$. $g(\cdot,\Delta)$ is differentiable for every $\Delta$ and $\nabla_\theta g(\theta,\Delta)$ s continuous on $\Theta\times \mathcal{X}$, the minimizer $\ktheta$ lies in the interior of $\Theta$, with non-zero $\nabla_{x_i} \hat{L}_n(\ktheta, X,Y)$ for all $i\in[n]$, then we have
\begin{align*}\label{eq:kernelaif}
\hat{\cI}(\cK)&=-\big(\sum_{i=1}^n K(x_i)K(x_i)^T+n\lambda I\big)^{-1}\\
&\Big(\sum_{k,i=1}^n\big( K(x_i)^T\ktheta+K(x_i)\ktheta^T-y_i\big)\cK_{x_i,x_k}\beta_k\Big).
\end{align*}
In the above formula, 
$$K(x_i)=\big(K(x_i,x_1),K(x_i,x_2),\cdots,K(x_i,x_n)\big)^T,$$
$$\cK_{x_i,x_k}=
 \Big( \frac{\partial K(x_i,x_1)}{\partial x_k} ,\cdots,\frac{\partial K(x_i,x_n)}{\partial x_k}\Big)^T
.$$
And the $z$-th coordinate of $\beta_k$ is
$$\beta_{k,z}=\frac{c_z^{q-1}}{(\sum_{k=1}^m c_z^q)^\frac{1}{p}}\text{sgn}\Big(\nabla_{x_k}\hat{L}_n(\ktheta,z)\Big)\bE_{x\sim\hat{\bP}_{x}}\|x\|_p,$$
with $c_z=|\nabla_{x_k}\hat{L}_n(\ktheta,z)|$, where $\nabla_{x_k}\hat{L}_n(\ktheta,z)$ is short for the $z$-th coordinate of $\nabla_{x_k}\hat{L}_n(\ktheta,X,Y)$:
$$\nabla_{x_k}\hat{L}_n(\theta,X,Y)=\frac{2}{n}\sum_{i=1}^n\Big(K(x_i)^T\ktheta-y_i\Big)\cK^T_{x_i,x_k}\ktheta.$$
\end{corollary}

\paragraph{Neural tangent kernels}
The intimate connection between kernel regression and overparametrized two-layer neural networks has been studied in the literature, see \cite{jacot2018neural, du2018gradient}. In this section, we are going to apply Corollary~\ref{col:kernel} to the two-layer neural networks in the over-parametrized setting.

Specifically, we consider a two-layer ReLU activated neural network with $b$ neurons in the hidden layer:
$$
f_{W,a}(x)=\frac{1}{\sqrt b}\sum_{r=1}^b a_r \sigma(w_r^T x),
$$
where $x \in \R^m$ denotes the input, $w_1,...,w_b \in \R^m$ are weight vectors in the first layer, $a_1,...,a_b\in \R$ are weights in the second layer. Further we denote $W = (w_1,...,w_b) \in \R^{m\times b}$ and $a = (a_1,...,a_m)^T\in \R^m$.

Suppose we have $n$ samples $S=\{(x_i,y_i)\}^n_{i=1}$ and assume $\|x_i\|_2=1$ for simplicity. We train the neural network by randomly initialized gradient descent on the quadratic loss over data $S$. In particular, we  initialize the parameters randomly:
$w_r\sim N(0,\kappa^2I)$, $a_r \sim  U(-1,1)$, for all $r \in [m],$  then Jacot et al. [2018] showed that, such a resulting network converges to the solution produced by the kernel regression with the so called Neural Tangent Kernel (NTK) matrix:
$$
NTK=\left[\frac{x_i^\top x_j (\pi-\arccos(x_i^\top x_j))}{2\pi}\right]_{i,j\in[n]}.
$$
In Figure \ref{fig:mnist}, we experimentally demonstrate the effectiveness of the approximation of AIF in kernel regressions with neural tangent kernel on MNIST. The estimation is based on the average of randomly drawn $300$ examples from MNIST for $10$ times.

\subsection{Distributional adversarial influence function}\label{subsec:daif}
Another popular way to formulate adversarial attack is through distributional robust optimization (DRO), where instead of perturbing $x$ with certain distance, one perturbs $(x,y)$ in a distributional sense. For a model $\cM$, the corresponding distributional robust optimization with respect to $u$-Wasserstein distance $W_u$ for $u\in[1,\infty)$ regarding $l_p$-norm is formulated as: 
\begin{align*}
\min_{\theta^\cM}~OPT(\varepsilon;\theta^\cM),
\end{align*}
where $OPT(\varepsilon;\theta^\cM)$ is defined as
$$OPT(\varepsilon;\theta^\cM):=\max_{\tilde{\bP}_{x,y}:W_u(\tilde{\bP}_{x,y},\bP_{x,y})\leq \varepsilon} \bE_{\tilde{\bP}_{x,y}}l(\theta^{\cM},x,y;\cM).$$
Here, $W_u(\mathcal{D},\mathcal{\tilde{D}})=\inf\{\int \|x-y\|_p^ud\gamma(x,y):\gamma\in\Pi(\mathcal{D},\mathcal{\tilde{D}})\}^{1/u}$ for two distributions $\mathcal{D},\tilde{\mathcal{D}}$, and $\Pi(\mathcal{D},\mathcal{\tilde{D}})$ are couplings of $\mathcal{D},\mathcal{\tilde{D}}$. However, it is not clear whether 
\begin{equation*}\label{eq:DRO}
\detheta:=\argmin_{\theta^\cM}~OPT(\varepsilon \bE_{\hat{\bP}_{x}}\|x\|_p;\theta^\cM),
\end{equation*}
is well-defined since the optimizer may not be unique. Moreover, the corresponding sample version of the optimizer $\detheta$ is not easy to obtain via regular optimization methods if we just replace the distribution $\bP_{x,y}$ by its empirical distribution since it is hard to get the corresponding worst form of $\tilde{\bP}_{x,y}$. As a result, we focus on defining empirical distributional adversarial influence function for a special approximation algorithm and state its limit. Interested readers are refered to the following result in \cite{staib2017distributionally} and \cite{gao2016distributionally} to properly find an approximation for $\tilde{\bP}_{x,y}$.
\begin{lemma}[A variation of Corollary 2(iv) in \cite{gao2016distributionally}]\label{lemma:approx}
Suppose for all $y$, $l(\theta^{\cM},x,y;\cM)$ is L-Lipschitz as a function of x. Define 
\begin{align*}
EMP(\varepsilon)& := \max_{\delta_1,\cdots,\delta_n}\frac{1}{n}\sum_{i=1}^n l(\theta^{\cM},x_i+\delta_i,y_i,\cM),~~s.t.~~(\frac{1}{n}\sum_{i=1}^n \|\delta_i\|^u_p)^{1/u}\leq \varepsilon.
\end{align*}
Then, we have $EMP(\varepsilon)\geq OPT(\varepsilon;\theta^\cM)-LD/n$ where $D$ bounds the maximum deviation of a single point.
\end{lemma}
Lemma \ref{lemma:approx} provides a direction to define an \textbf{algorithm dependent} empirical DAIF $\hat{\cI}^{DRO}(\cM)$. We define $\hat{\cI}^{DRO}(\cM)$ similarly as before.
%
For a given model $\cM$, the corresponding empirical distributional adversarial influence function is defined as  
\begin{align*}
&\hat{\cI}^{DRO}(\cM):=\frac{d\hdetheta}{d\varepsilon }\big|_{\varepsilon=0+},~~s.t.~ ~\hdetheta \in\argmin_{\theta^{\cM}\in\Theta} EMP\Big(\varepsilon\bE_{\hat{\bP}_{x}}\|x\|_p\Big).
\end{align*}
We use $\in \argmin$ here since there may not be a unique minimizer, but the limit $\hat{\cI}^{DRO}(\cM)$ is still unique and well-defined. Similarly, we can provide a closed form of distributional adversarial influence function. 
\begin{theorem}\label{thm:DAIF}
Under the settings of Theorem \ref{thm:firstorder},
\begin{equation}\label{eq:daif}
\hat{\cI}^{DRO}(\cM)=-\hat{H}_{\htheta}^{-1}\varrho n^{\frac{1-u}{u}} ,
\end{equation}
where $\varrho = \nabla_{x,\theta} l(\htheta,x_J,y_J)\bE_{\hat{\bP}_{x}}\|x\|_p\phi_J$ and $\phi_{i}$ is defined as in Theorem \ref{thm:firstorder}, $J$ is the index:
$$J=\argmax_i \|\nabla_{x} L(\htheta,x_i,y_i)\|_q.$$
\end{theorem}
We remark here from Eq. \ref{eq:daif}, we can see that if $u>1$, more training data will result in a smaller norm of $\hat{\cI}^{DRO}(\cM)$ since there is a factor $n^{(1-u)/u}$.

\vspace{-0.4cm}
 \section{Conclusions and Future Work}\label{sec:con}

To achieve adversarial robustness, robust optimization has been widely used in the training of deep neural networks, while their theoretical aspects are under-explored. In this work we first propose the AIF to quantify the influence of robust optimization theoretically. The proposed AIF is then used to efficiently approximate the model sensitivity, which is usually NP-hard to compute in practice. We then apply the AIF to study the relationship between model sensitivity and model complexity. Moreover, the AIF is applied to randomized smoothing and found that adding noise to the input during training would help reduce the model sensitivity. Further, the theories are extended to the kernel regression models and distributional robust optimization. {Based on the newly introduced tool AIF, we suggest two main directions for future research.

First, we can study how to use AIF to select model with the greatest adversarial robustness. Due to the computational effectiveness of AIF, it is a natural idea to use AIF for model selection. Such an idea can be used for tuning parameter selection in statistical models such as high-dimensional regression and factor analysis, and further extended to the neural network depth and width selection. 

Second, AIF can be extended to study more phenomena in adversarial training. For instance, the relationship between low-dimensional representations and adversarial robustness. Recently, \cite{lezama2018ole, sanyal2018learning} empirically observed that using learned low-dimensional representations as the input in neural networks is substantially more adversarially robust, but a theoretical exploration of this phenomenon is still lacking. 
}

 \section{Acknowledgments}\label{sec:acknowledge}
This work is in part supported by NSF award 1763665 and NSF DMS-2015378.







\bibliography{aif_cite}
\bibliographystyle{plain}

\newpage
\appendix
\noindent\textbf{\Large Appendix}
\section{Omitted Proofs}

\subsection{Proof of Theorem \ref{thm:firstorder}}
In order to prove Theorem \ref{thm:firstorder}, let us first state the Danskin theorem. 
\begin{lemma}[Danskin]\label{lemma:Danskin}
 Let $\mathcal{B}$ be nonempty compact topological space and $h:\mathbb{R}^d\times\mathcal{B}\rightarrow \mathbb{R}$ be such that $h(\cdot,\delta)$ is differentiable for every $\delta\in\mathcal{B}$ and $\nabla_{\theta} h(\theta,\delta)$ is continuous on $\mathbb{R}^d\times\mathcal{B}$. Also, let $\delta^*(\theta)=\{\delta\in\argmax_{\delta\in\mathcal{B}}h(\theta,\delta)\}$.\\
Then, the corresponding max-function
$$\varsigma(\theta)=\max_{\delta\in\mathcal{B}}h(\theta,\delta)$$
is locally Lipschitz continuous, directionally differentiable, and its directional derivatives satisfy 
$$\varsigma'(\theta,r)=\sup_{\delta\in\delta^{*}(\theta)}r^T\nabla_\theta h(\theta,\delta).$$
In particular, if for some $\theta\in\mathbb{R}^d$ the set $\delta^*(\theta)=\{\delta^*_\theta\}$ is a singleton, the max-function is differentiable at $\theta$ and 
$$\nabla\varsigma(\theta)=\nabla_{\theta}h(\theta,\delta^*_\theta).$$
\end{lemma}
By this lemma, we can easily obtain the following lemma:
\begin{lemma}\label{lemma:expansion}
For any $\tilde{\theta}$ that minimize $\varsigma(\theta)$ and lying in the interior, we can obtain
$$\nabla_{\theta}h(\tilde{\theta},\delta)=0.$$
\end{lemma}
\begin{proof}
Since $\tilde{\theta}$ minimizes $\varsigma(\theta)$ and lies in the interior of $\Theta$ , we can obtain
$$\varsigma'(\tilde{\theta},r)=0$$
for any direction vector $r$.

If there is a $\delta\in\delta^*(\tilde{\theta})$, such that $\nabla_{\theta}h(\tilde{\theta},\delta)\neq 0$, then we take $r=\nabla_{\theta}h(\tilde{\theta},\delta)/\|\nabla_{\theta}h(\tilde{\theta},\delta)\|$, we have 
$$\varsigma'(\tilde{\theta},r)=\sup_{\delta\in\delta^{*}(\tilde{\theta})}r^T\nabla_\theta h(\tilde{\theta},\delta)\geq\|\nabla_{\theta}h(\tilde{\theta},\delta)\|_2>0,$$
which is contradictory to the fact $\varsigma'(\tilde{\theta},r)=0$.
\end{proof}
\noindent \textbf{[Proof of Theorem 1]} Now we are ready to give the formal proof. In order for simplicity, we here use $L(\theta^{\cM},x,y)$ instead of $L(\theta^{\cM},x,y,\cM)$.
With lemma \ref{lemma:expansion}, we can obtain that 
$$\frac{1}{n}\sum_{i=1}^n\nabla_\theta L(\theta^{\cM},x_i+\delta_i,y_i)|_{\theta^{\cM}=\hetheta}=0.$$
With Taylor expansion and under the assumption of Lemma \ref{lemma:expansion}, we can obtain 
$$0=\frac{1}{n}\sum_{i=1}^n[\nabla_\theta L(\hetheta,x_i,y_i)+\nabla_{x,\theta} L(\hetheta,x_i,y_i)\delta_i+O(\|\delta\|^2_2)].$$
Here the assumption of compactness and continuity can help us to write the remainder $\frac{1}{2}\delta_i^TH_{\tilde{\theta}}\delta_i $ into $O(\|\delta_i\|^2_2)$ since we can bound every entry of $H_{\tilde{\theta}}$. We use the same property repeatedly and will not reiterate it.

Now, let us perform taylor expansion on $\nabla_\theta L(\hetheta,x_i,y_i)$ and $\nabla_{x,\theta} L(\hetheta,x_i,y_i)$.
\begin{equation*}
\nabla_\theta L(\hetheta,x_i,y_i)=\nabla_\theta L(\htheta,x_i,y_i)+\nabla^2_\theta L(\htheta,x_i,y_i)(\hetheta-\htheta)+O(\|\hetheta-\htheta\|^2_2)
\end{equation*}
and 
$$\nabla_{x,\theta} L(\hetheta,x_i,y_i)=\nabla_{x,\theta} L(\htheta,x_i,y_i)+O(\|\hetheta-\htheta\|_2).$$
By simple algebra, 
\begin{align*}
\hetheta-\htheta+O(\|\hetheta-\htheta\|^2_2)=&\big(-\frac{1}{n}\sum_{i=1}^n\nabla^2_\theta L(\htheta,x_i,y_i)\big)^{-1}\Big(\frac{1}{n}\sum_{i=1}^n\nabla_{x,\theta} L(\htheta,x_i,y_i)\delta_i\\
&+\|\hetheta-\htheta\|_2\|\delta_i\|_2\Big).
\end{align*}
We know if we divided $\varepsilon$ on both sides, we know when $\varepsilon$ goes to 0, the limit of the right handside exists if we assume the limit of $\lim_{\varepsilon\rightarrow 0}\delta_i/\varepsilon$ exist (notice $\delta_i$ is a implicit function of $\varepsilon$). Thus, $\|\hetheta-\htheta\|/\varepsilon$ cannot goes to infinity as $\varepsilon$ goes to 0. In orther words, AIF must exist. 

Now the only thing left is to prove $\lim_{\varepsilon\rightarrow 0}\delta_i/\varepsilon$ exist. We prove that 
$$\lim_{\varepsilon\rightarrow 0}\frac{\delta_i}{\varepsilon}=\gamma_i,$$
where 
$$\gamma_{i,k}=\frac{b_k^{q-1}}{(\sum_{k=1}^m b_k^q)^\frac{1}{p}}\text{sgn}\Big(\frac{\partial}{\partial x_{\cdot,k}}L(\htheta,x_i,y_i)\Big),$$
with $b_k=|\frac{\partial}{\partial x_{\cdot,k}}L(\htheta,x_i,y_i)|$. By H$\ddot{o}$lder inequality, we know 
$$|\nabla_{x} L(\htheta,x_i,y_i)\cdot\delta_i|\leq \varepsilon\|\nabla_{x} L(\htheta,x_i,y_i\|_q$$
the equality holds if and only if $\delta_i=\varepsilon\gamma_i$ .
Since 
$$ L(\hetheta,x_i+\delta_i,y_i)= L(\hetheta,x_i,y_i)+\nabla_{x} L(\hetheta,x_i,y_i)\delta_i+O(\|\delta_i\|^2_2),$$
we know the reminder is ignorable 
$$\frac{O(\|\delta_i\|^2_2)}{ |L(\hetheta,x_i,y_i)\delta_i|}\rightarrow 0$$
as $\varepsilon$ goes to $0$.
So, we must have 
$$\lim_{\varepsilon\rightarrow 0}\frac{\delta_i}{\varepsilon}=\gamma_i.$$
As a result, 
\begin{equation*}
\hat{\cI}(\cM)=-H_{\htheta}^{-1}\Phi
\end{equation*}
as described in the theorem.

\subsection{Proof of Theorem \ref{linearandquadratic}}

Let us first compute the AIF for linear models. 

Specifically, let us consider the regression setting $(x_i, y_i)\in\mathbb R^{d}\times \mathbb R$ are $i.i.d.$ draws from a joint distribution $P_{x,y}$, for $i=1,2,...,n$. Note that we don't assume linear relationship, but the linear regression model tries to find the best linear approximation by solving $$
\hat\theta=\arg\min_{\theta} \frac{1}{n}\sum_{i=1}^nl(\theta, x_i, y_i):=\arg\min_{\theta} \frac{1}{2n}\sum_{i=1}^n(y_i-\theta^T x_i)^2,
$$
where we use $l(\theta, x_i, y_i)=\frac{1}{2}(y_i-\theta^T x_i)^2$ as the loss function. 

Further, let us define $$
\theta^*=\arg\min_{\theta} \E_{P_{x,y}}[\frac{1}{2}(Y-\theta^T X)^2],
$$
denoting the best population linear approximation to $Y$.

When the true model is $Y=X^\top\beta_1^*+(X^\top\beta_2^*)^2+\xi$, and $X\sim \cN(0,\sigma_x^2 I)$, we have $$
\theta^*=(\E[XX^\top])^{-1}\E[XY]=(\E[XX^\top])^{-1}(\E[XX^\top\beta_1^*]+\E[X(X^\top\beta_2^*)^2])=\beta_1^*.
$$

Further, denote $\epsilon_i=y_i-\theta^{*\top}x_i$, and 
 $$
\hat\theta=\arg\min_{\theta} \frac{1}{2n}\sum_{i=1}(y_i-\theta^T x_i)^2,
$$
and we have $\|\hat\theta-\theta^*\|_2=O_p(\sqrt{\frac{m}{n}})$.

By definition, for $k\in[m]$, 
\begin{align*}
b_k=&|\frac{\partial}{\partial x_{\cdot,k}}l(\htheta,x_i,y_i,\mathcal M)|=|y_i- \hat\theta^\top x_i|\cdot |\hat\theta_k|,
\end{align*}
and therefore, by letting $ p=q=2$, in Eqn (5) of Theorem 4.1,
\begin{align*}
\psi^i_k&=\frac{b_k^{q-1}}{(\sum_{k=1}^d b_k^q)^\frac{1}{p}}\text{sgn}(\frac{\partial}{\partial x_{\cdot,k}}l(\htheta,x_i,y_i,\mathcal M))=\frac{b_k}{(\sum_{k=1}^d b_k^2)^{1/2}}\text{sgn}((y_i- \hat\theta^\top x_i)\cdot \hat\theta_k)\\
&=\frac{(y_i- \hat\theta^\top x_i)\cdot \hat\theta_k}{|y_i- \theta^\top x_i|\cdot \|\hat\theta\|_2}=\frac{\hat\theta_k}{\|\hat\theta\|}\cdot\text{sgn}(y_i-\hat\theta^\top x_i).
\end{align*}

As a result
$$
\phi_i=(\psi^i_1,\psi^i_2,\cdots,\psi^i_m)^T=\text{sgn}(y_i-\hat\theta^\top x_i)\cdot\frac{1}{\|\hat\theta\|}\cdot\hat\theta,
$$
and
 \begin{align*}
\frac{\Phi}{\E_{x\sim\hat P_x}\|x\|_2}=&\frac{1}{n}\sum_{i=1}^n\nabla_{x,\theta} l(\htheta,x_i,y_i,\mathcal M)\phi_i=\frac{1}{n}\sum_{i=1}^n[(\hat\theta^\top x_i-y_i)\cdot I_d+\hat\theta x_i^\top]\cdot\text{sgn}(y_i-\hat\theta^\top x_i)\cdot\frac{1}{\|\hat\theta\|}\cdot\hat\theta\\
=&\frac{1}{n\|\hat\theta\|}\sum_{i=1}^n[(\hat\theta^\top x_i-y_i)\cdot \hat\theta+\hat\theta x_i^\top\hat\theta]\cdot\text{sgn}(y_i-\hat\theta^\top x_i)\\
=&-\frac{1}{n\|\hat\theta\|}\sum_{i=1}^n (y_i\cdot \hat\theta)\cdot\text{sgn}(y_i-\hat\theta^\top x_i)\\
=&-\frac{\hat\theta}{n\|\hat\theta\|}\sum_{i=1}^n y_i \cdot\text{sgn}(y_i-\hat\theta^\top x_i)\\
=&-\frac{\hat\theta}{n\|\hat\theta\|}\sum_{i=1}^n (\epsilon_i+\theta^{*\top} x_i) \cdot\text{sgn}(y_i-\hat\theta^\top x_i)\\
=&-\frac{\hat\theta}{n\|\hat\theta\|}\sum_{i=1}^n\theta^{*\top} x_i \cdot\text{sgn}(y_i-\hat\theta^\top x_i)-\frac{\hat\theta}{n\|\hat\theta\|}\sum_{i=1}^n \epsilon_i \cdot\text{sgn}(y_i-\hat\theta^\top x_i)\\
=&-\frac{\hat\theta}{n\|\hat\theta\|}\sum_{i=1}^n\theta^{*\top} x_i \cdot\text{sgn}(\epsilon_i)-\frac{\hat\theta}{n\|\hat\theta\|}\sum_{i=1}^n \epsilon_i \cdot\text{sgn}(\epsilon_i)\\
&+\frac{\hat\theta}{n\|\hat\theta\|}\sum_{i=1}^n\theta^{*\top} x_i \cdot (\text{sgn}(\epsilon_i)-\text{sgn}(\epsilon_i-(\hat\theta-\theta^*)^{\top} x_i))\\
&+\frac{\hat\theta}{n\|\hat\theta\|}\sum_{i=1}^n \epsilon_i \cdot(\text{sgn}(\epsilon_i)-\text{sgn}(\epsilon_i-(\hat\theta-\theta^*)^\top x_i)).
\end{align*} 

$$
\Pb(\text{sgn}(\epsilon_i)\neq\text{sgn}(\epsilon_i-(\hat\theta-\theta^*)^\top x_i)\le\Pb(|\epsilon|\le |(\hat\theta-\theta^*)^\top x_i|)=O(\sqrt\frac{1}{n})=o(1).
$$

Recall that $\epsilon_i=y_i-\theta^{*\top}x_i=(x_i^\top\beta_2^*)^2+\xi_i$, then we obtain $x_i\text{sgn}((x_i^\top\beta_2^*)^2+\xi_i)\stackrel{d}{=}-x_i\text{sgn}((x_i^\top\beta_2^*)^2+\xi_i)$. As a result,  we have {\color{black}$\E[x_i\text{sgn}(\epsilon_i)]=\E[x_i\text{sgn}((x_i^\top\beta_2^*)^2+\xi_i)]=0$}, yielding $$
\frac{1}{n}\sum_{i=1}^n\theta^{*\top} x_i \cdot\text{sgn}(\epsilon_i)=O_p(\frac{1}{\sqrt n}).
$$

Then, we have $$
\frac{\Phi}{\E_{x\sim\hat P_x}\|x\|_2}=-\frac{\hat\theta}{\|\hat\theta\|}(\frac{1}{n}\sum_{i=1}^n |\epsilon_i|+O_p(\frac{1}{\sqrt n}))=-\frac{\hat\theta}{\|\hat\theta\|}(\E |\epsilon_i|+O_p(\frac{1}{\sqrt n}))
$$

Moreover, the Hessian matrix 
$$
H_{\theta}(X^e,Y^e)=1/{n'}\sum_{i=1}^{n'}\nabla^2_{\theta} l(\htheta,x_i^e,y_i^e;\mathcal{A})=\frac{1}{n}X^{e\top}X^{e}=\sigma_x^2 I+O_p(\sqrt\frac{m}{n}).
$$
Then \begin{align*}
\hat S_{\epsilon}(\mathcal{L})&=\Phi^\top H^{-1}_{\theta}(X^e,Y^e)\Phi=(\E_{x\sim\hat P_x}\|x\|_2)^2(\E|\epsilon_i|+O_p(\frac{1}{\sqrt n}))^2 \frac{\hat\theta^\top(\sigma_x^2 I+O_p(\sqrt\frac{m}{n}))^{-1}\hat\theta}{\|\hat\theta\|^2}\\
&=(\E_{x\sim\hat P_x}\|x\|_2)^2\cdot[(\E|\epsilon_i|)^2+O_p(\frac{1}{\sqrt n})]\cdot (\sigma_x^{-2} +O_p(\sqrt\frac{m}{n})).
\end{align*}

Then, let us consider the quadratic basis of the regression setting $(x_i, y_i)\in\mathbb R^{m}\times \mathbb R$ are $i.i.d.$ draws from a joint distribution $P_{x,y}$, for $i=1,2,...,n$. Suppose we use the basis $ v( x)=(v_1( x),...,v_d( x))=(x_1,...,x_m,x_1^2/2,...,x_m^2/2, \{x_jx_k\}_{j<k}),$ to approximate $y$, and try to solve
 $$
\hat\theta=\arg\min_{\theta\in\R^m} \frac{1}{n}\sum_{i=1}^nl(\theta, x_i, y_i):=\arg\min_{\theta} \frac{1}{2n}\sum_{i=1}^n(y_i-\theta^T  v(x_i))^2.
$$

Further, let us define $$
\theta^*=\arg\min_{\theta} \E_{P_{x,y}}[\frac{1}{2}(Y-\theta^T  v(X))^2],
$$
denoting the best population linear approximation to $Y$.

Denote $\epsilon_i=y_i-\theta^{*\top} v(x_i)$. Since the true model is $Y=X^\top\beta_1^*+(X^\top\beta_2^*)^2+\xi$, and $X\sim N(0,\sigma_x^2 I)$, then $\epsilon_i=\xi_i$ and we have $\E[\text{sgn}(\epsilon_i) x_i]=0$.

Further, denote 
$$
\hat\theta=\arg\min_{\theta} \frac{1}{2n}\sum_{i=1}(y_i-\theta^T  v(x_i))^2.
$$

We have $\|\hat\theta- \theta^*\|=O_p(\sqrt\frac{m}{n})$.

By definition, for $k\in[m]$, 
\begin{align*}
b_k=&|\frac{\partial}{\partial x_{\cdot,k}}l(\htheta,x_i,y_i)|=|y_i- \hat\theta^\top  v(x_i)|\cdot |\hat\theta^\top\frac{\partial}{\partial x_{\cdot,k}} v(x_i)|=|y_i- \hat\theta^\top  v(x_i)|\cdot |\hat\theta^\top\frac{\partial}{\partial x} v(x_i)e_k|.
\end{align*}

Therefore, by letting ${\color{black} p=q=2}$ in Eqn (5) of Theorem 4.1,
\begin{align*}
\psi^i_k&=\frac{b_k^{q-1}}{(\sum_{k=1}^d b_k^q)^\frac{1}{p}}\text{sgn}(\frac{\partial}{\partial x_{\cdot,k}}l(\htheta,x_i,y_i))=\frac{b_k}{(\sum_{k=1}^d b_k^2)^{1/2}}\text{sgn}(\frac{\partial}{\partial x_{\cdot,k}}l(\htheta,x_i,y_i))\\
&=\frac{(\hat\theta^\top  v(x_i)-y_i)\cdot \hat\theta^\top\frac{\partial}{\partial x} v(x_i)e_k}{|y_i- \hat\theta^\top v( x_i)|\cdot \|\hat\theta^\top\frac{\partial}{\partial x} v(x_i)\|_2}=\frac{\hat\theta^\top\frac{\partial}{\partial x} v(x_i)e_k}{\|\hat\theta^\top\frac{\partial}{\partial x} v(x_i)\|}\cdot\text{sgn}(\hat\theta^\top  v(x_i)-y_i).
\end{align*}

As a result
$$
\phi_i^\top=( \psi^i_1, \psi^i_2,\cdots, \psi^i_m)=\text{sgn}(\hat\theta^\top  v(x_i)-y_i)\cdot\frac{1}{\|\hat\theta^\top\frac{\partial}{\partial x} v(x_i)\|}\cdot\hat\theta^\top\frac{\partial}{\partial x} v(x_i),
$$
and
$$
\nabla_{x} l(\hat\theta,x_i,y_i)=(\hat\theta^\top v(x_i)-y_i)\cdot(\frac{\partial}{\partial x} v(x_i))^\top\hat\theta
$$
$$
\nabla_{x,\theta} l(\hat\theta,x_i,y_i)= v(x_i)\hat\theta^\top\frac{\partial}{\partial x} v(x_i)+(\hat\theta^\top v(x_i)-y_i)\cdot \frac{\partial}{\partial x} v(x_i)
$$

Then
 \begin{align*}
\frac{\Phi}{\E_{x\sim\hat P_x}\|x\|_2}=&\frac{1}{n}\sum_{i=1}^n\nabla_{x,\theta} l(\htheta,x_i,y_i)\phi_i\\
=&\frac{1}{n}\sum_{i=1}^n[(\hat\theta^\top  v(x_i)-y_i)\cdot \frac{\partial}{\partial x} v(x_i)+ v(x_i)\hat\theta^\top\frac{\partial}{\partial x} v(x_i)]\cdot\text{sgn}(\hat\theta^\top  v(x_i)-y_i)\\
&\cdot\frac{1}{\|\hat\theta^\top\frac{\partial}{\partial x} v(x_i)\|}\cdot(\frac{\partial}{\partial x} v(x_i))^\top \hat\theta\\
=&\frac{1}{n\|\hat\theta^\top\frac{\partial}{\partial x} v(x_i)\|}\sum_{i=1}^n[(\hat\theta^\top v( x_i)-y_i)\cdot \frac{\partial}{\partial x} v(x_i)(\frac{\partial}{\partial x} v(x_i))^\top \hat\theta+  v(x_i)\\
&\cdot\|\hat\theta^\top\frac{\partial}{\partial x} v(x_i)\|^2 ]\cdot\text{sgn}(\hat\theta^\top  v(x_i)-y_i)\\
=&\frac{1}{n\|\hat\theta^\top\frac{\partial}{\partial x} v(x_i)\|}\sum_{i=1}^n |\hat\theta^\top v( x_i)-y_i|\cdot \frac{\partial}{\partial x} v(x_i)(\frac{\partial}{\partial x} v(x_i))^\top \hat\theta\\
&+\frac{1}{n}\sum_{i=1}^n  v(x_i)\cdot\|\hat\theta^\top\frac{\partial}{\partial x} v(x_i)\|\cdot\text{sgn}(\hat\theta^\top  v(x_i)-y_i)\\
=&\E[\frac{ |\hat\theta^\top v( x_i)-y_i|}{\|\hat\theta^\top\frac{\partial}{\partial x}\bm v(x_i)\|}\frac{\partial}{\partial x}v(x_i)(\frac{\partial}{\partial x}v(x_i))^\top\hat\theta]+O_p(\sqrt\frac{m^2}{n})
\end{align*} 

Then we have 
\begin{align*}
&\|\E[\frac{ |\hat\theta^\top v( x_i)-y_i|}{\|\hat\theta^\top\frac{\partial}{\partial x}\bm v(x_i)\|}\frac{\partial}{\partial x}\bm v(x_i)(\frac{\partial}{\partial x}\bm v(x_i))^\top\hat\theta]\|^2\\
\le&\E[\|\frac{|\hat\theta^\top v( x_i)-y_i|}{\|\hat\theta^\top\frac{\partial}{\partial x} v(x_i)\|}\frac{\partial}{\partial x} v(x_i)(\frac{\partial}{\partial x} v(x_i))^\top\hat\theta\|^2]\\
\le&\E[\|(\frac{\partial}{\partial x} v(x_i))^\top\frac{\partial}{\partial x} v(x_i)\|_2\cdot |\hat\theta^\top v( x_i)-y_i|^2]\\
\le&\E[\|(\frac{\partial}{\partial x} v(x_i))^\top\frac{\partial}{\partial x} v(x_i)\|_2\cdot |\theta^{*\top} v( x_i)-y_i|^2]+O_p(\sqrt\frac{m^2}{n})\\
\le&\E[\lambda_{\max}(\frac{\partial}{\partial x} v(x_i))^\top\frac{\partial}{\partial x} v(x_i))]\cdot((\E|\xi_i|)^2+O_p(\sqrt\frac{1}{n}))+O_p(\sqrt\frac{m^2}{n})
\end{align*}

By similar argument, we have \begin{align*}
\|\E[\frac{1}{\|\hat\theta^\top\frac{\partial}{\partial x} v(x_i)\|}\frac{\partial}{\partial x} v(x_i)(\frac{\partial}{\partial x} v(x_i))^\top\hat\theta]\|^2\ge&\E[\lambda_{\min}(\frac{\partial}{\partial x} v(x_i))^\top\frac{\partial}{\partial x} v(x_i))]\cdot((\E|\epsilon_i|)^2\\
&+O_p(\sqrt\frac{1}{n}))+O_p(\sqrt\frac{m^2}{n}).
\end{align*}

Recall that $ v(x)=(x_1,...,x_m,x_1^2/2,...,x_m^2/2, \{x_jx_k\}_{j<k}),$ and the quadratic term in the true model is $(\beta_2^{*\top} x)^2$ (so then $\E|\epsilon|$ is easy to compute), 
then $$
\frac{\partial}{\partial x} v(x)= (I_m, \text{diag}(x_1,...,x_m),Perm(x_ix_j))^\top=\left[
	\begin{array}{c}
	I_m \\  
	D_x\\
	Perm(x_ix_j)
	\end{array}\right]\in\mathbb{R}^{(m^2+m)\times m},
$$
where $D_x=\text{diag}(x_1,...,x_d)$, and $Perm(x_ix_j)\in\mathbb{R}^{m\times (m^2-m)}$ with each column being $x_j e_k+x_k e_j$ for $1\le j<k\le m$.

Then we have $$
(\frac{\partial}{\partial x} v(x_i))^\top\frac{\partial}{\partial x} v(x_i)=I_m+D_{Q},
$$
where $(D_{Q})_{jj}=(x_1^2+...+x_m^2)$, $(D_{Q})_{jk}=x_jx_k$. As a result, $D_{Q}= x  x^\top+ (x_1^2+...+x_m^2)I_m-D_x^2$
$$
\inf_{v:\|v\|=1} v^\top( x  x^\top+ (x_1^2+...+x_d^2)I_d)v =  (x_1^2+...+x_d^2)+\inf_v  (x^\top v)^2
$$

Therefore, $$
1+(x_1^2+...+x_d^2)-\max x_j^2\le\lambda_{\min}(I_d+D_{Q})\le \lambda_{\max}(I_d+D_{Q})\le 1+2(x_1^2+...+x_d^2).
$$

Moreover, \begin{align*}
H_{\theta}(X^e,Y^e)&=\frac{1}{n}\sum_{i=1}^n v(x_i) v(x_i)^\top=E[v(x_i) v(x_i)^\top]+O_p(\sqrt\frac{m^2}{n})\\
&=diag(\sigma_x^2 I_m, \frac{3}{4} \sigma_x^4 I_{m}, \sigma_x^4 I_{m(m-1)})+O_p(\sqrt\frac{m^2}{n})
\end{align*}

Then

\begin{align*}
\hat S_{\epsilon}(\mathcal{Q})\le&\frac{1}{\max\{\sigma_x^2, \sigma_x^4\}}\E[\lambda_{\max}(\frac{\partial}{\partial x}\bm v(x_i))^\top\frac{\partial}{\partial x}\bm v(x_i))]\cdot((\E|\xi_i|)^2+O_p(\sqrt\frac{1}{n}))+O_p(\sqrt\frac{m^2}{n})\\
\le&\frac{1}{\max\{\sigma_x^2, \sigma_x^4\}}\E[1+2(x_1^2+...+x_m^2)]\cdot((\E|\xi_i|)^2+O_p(\sqrt\frac{1}{n}))+O_p(\sqrt\frac{m^2}{n})\\
\le&\frac{1}{\max\{\sigma_x^2, \sigma_x^4\}}(1+2m\sigma_x^2)\cdot((\E|\xi_i|)^2+O_p(\sqrt\frac{1}{n}))+O_p(\sqrt\frac{m^2}{n})
\end{align*}

Similarly,
 \begin{align*}
\hat S_{\epsilon}(\mathcal{Q})\ge&\frac{1}{\min\{\sigma_x^2, \frac{3}{4}\sigma_x^4\}}\E[\lambda_{\min}(\frac{\partial}{\partial x} v(x_i))^\top\frac{\partial}{\partial x} v(x_i))]\cdot((\E|\xi_i|)^2+O_p(\sqrt\frac{1}{n}))+O_p(\sqrt\frac{m^2}{n})\\
\ge&\frac{1}{\min\{\sigma_x^2, \frac{3}{4}\sigma_x^4\}}\E[1+(x_1^2+...+x_m^2)-\max x_j^2]\cdot((\E|\xi_i|)^2+O_p(\sqrt\frac{1}{n}))+O_p(\sqrt\frac{m^2}{n})\\
\ge&\frac{1}{\min\{\sigma_x^2, \frac{3}{4}\sigma_x^4\}}(1+(m-2\log m)\sigma_x^2)\cdot((\E|\xi_i|)^2+O_p(\sqrt\frac{1}{n}))+O_p(\sqrt\frac{m^2}{n})
\end{align*}

Recall that for linear model, 
\begin{align*}
\hat S_{\epsilon}(\mathcal{L})&=\Phi^\top H^{-1}_{\theta}(X^e,Y^e)\Phi=(\E_{x\sim\hat P_x}\|x\|_2)^2(\E|\epsilon_i|+O_p(\frac{1}{\sqrt n}))^2 \frac{\hat\theta^\top(\sigma_x^2 I+O_p(\sqrt\frac{m}{n}))^{-1}\hat\theta}{\|\hat\theta\|^2}\\
&=(\E_{x\sim\hat P_x}\|x\|_2)^2\cdot[(\E|\epsilon_i|)^2+O_p(\frac{1}{\sqrt n})]\cdot (\sigma_x^{-2} +O_p(\sqrt\frac{m}{n})).
\end{align*}


Since the true model is $y=\beta_1^{*\top} x+(\beta_2^{*\top} x)^2+\xi$ with $\xi\sim \cN(0,\sigma_\xi^2)$, and $x\sim \cN(0,\sigma_x^2 I_m)$, we have 
$$
(\E|\epsilon_i|)^2=(\E|(\beta_2^{*\top} x)^2+\xi|)^2\in[(\|\beta^*_2\|_2^2\sigma_x^2-\sqrt{\frac{2}{\pi}}\sigma_{\xi})^2,(\|\beta^*_2\|_2^2\sigma_x^2+\sqrt{\frac{2}{\pi}}\sigma_{\xi})^2].
$$

Therefore, when $\frac{1}{\sigma_x^2}(\|\beta^*_2\|_2^2\sigma_x^2-\sqrt{\frac{2}{\pi}}\sigma_{\xi})^2\ge \frac{1}{\max\{\sigma_x^2, \sigma_x^4\}}(1+2m\sigma_x^2)\cdot\frac{2}{\pi}\sigma_{\xi}^2$, that is, $
(\|\beta^*_2\|_2^2\sigma_x^2-\sqrt{\frac{2}{\pi}}\sigma_{\xi})^2\ge \frac{1+2m\sigma_x^2}{\max\{\sigma_x^2, 1\}}\cdot\frac{2}{\pi}\sigma_{\xi}^2,
$
$$
\hat\Delta(\mathcal{L})\ge\hat\Delta(\mathcal{Q})+O_p(\sqrt\frac{m^2}{n}).
$$
On the other hand, $\frac{1}{\sigma_x^2}(\|\beta^*_2\|_2^2\sigma_x^2+\sqrt{\frac{2}{\pi}}\sigma_{\xi})^2\le  \frac{1}{\min\{\sigma_x^2, \frac{3}{4}\sigma_x^4\}}(1+m\sigma_x^2-2\sigma_x^2\cdot\log m)\cdot\frac{3}{2\pi}\sigma_\epsilon^2$, that is, $(\|\beta^*_2\|_2^2\sigma_x^2+\sqrt{\frac{2}{\pi}}\sigma_{\xi})^2\le  \frac{1}{\min\{1, \frac{3}{4}\sigma_x^2\}}(1+m\sigma_x^2-2\sigma_x^2\cdot\log m)\cdot\frac{3}{2\pi}\sigma_\epsilon^2$,
$$
\hat\Delta(\mathcal{L})\le\hat\Delta(\mathcal{Q})+O_p(\sqrt\frac{m^2}{n}).
$$

Then let us consider the case where $p=\infty, q=1$ in Eqn (5) of Theorem 4.1,
\begin{align*}
\nu^i_k&=\frac{b_k^{q-1}}{(\sum_{k=1}^d b_k^q)^\frac{1}{p}}\text{sgn}(\frac{\partial}{\partial x_{\cdot,k}}l(\htheta,x_i,y_i))=\text{sgn}(\frac{\partial}{\partial x_{\cdot,k}}l(\htheta,x_i,y_i))\\
&=\text{sgn}({\theta^\top\frac{\partial}{\partial x}\bm v(x_i)e_k})\cdot\text{sgn}(\theta^\top \bm v(x_i)-y_i).
\end{align*}
Then
 \begin{align*}
\Phi=&\frac{1}{n}\sum_{i=1}^n\nabla_{x,\theta} l(\htheta,x_i,y_i)\phi_i\\
=&\frac{1}{n}\sum_{i=1}^n[(\theta^\top  v(x_i)-y_i)\cdot \frac{\partial}{\partial x} v(x_i)+ v(x_i)\theta^\top\frac{\partial}{\partial x} v(x_i)]\cdot\text{sgn}(\theta^\top  v(x_i)-y_i)\cdot\text{sgn}((\frac{\partial}{\partial x}\bm v(x_i))^\top \theta)\\
=&\frac{1}{n}\sum_{i=1}^n |\theta^\top v( x_i)-y_i|\cdot \frac{\partial}{\partial x} v(x_i)\text{sgn}((\frac{\partial}{\partial x} v(x_i))^\top \theta)+\frac{1}{n}\sum_{i=1}^n  v(x_i)\cdot\|\theta^\top\frac{\partial}{\partial x} v(x_i)\|_1\cdot\text{sgn}(\theta^\top  v(x_i)-y_i)\\
=&\E[ |\theta^\top v( x_i)-y_i|\cdot \frac{\partial}{\partial x} v(x_i)\text{sgn}((\frac{\partial}{\partial x} v(x_i))^\top \theta)]+O_p(\frac{m}{\sqrt n})\\
\end{align*} 

By similar argument, since $\|\text{sgn}((\frac{\partial}{\partial x} v(x_i))^\top \theta\|=\sqrt{d}$ we have \begin{align*}
&\|\E [|\theta^\top v( x_i)-y_i|\cdot \frac{\partial}{\partial x} v(x_i)\text{sgn}((\frac{\partial}{\partial x} v(x_i))^\top \theta\|^2/d\\
\ge&\E[\lambda_{\min}(\frac{\partial}{\partial x} v(x_i))^\top\frac{\partial}{\partial x} v(x_i))]\cdot((\E|\epsilon_i|)^2+O_p(\sqrt\frac{1}{n}))+O_p(\sqrt\frac{m^2}{n}).
\end{align*}
\begin{align*}
&\|\E [|\theta^\top v( x_i)-y_i|\cdot \frac{\partial}{\partial x}\bm v(x_i)\text{sgn}((\frac{\partial}{\partial x} v(x_i))^\top \theta\|^2/d\\
\le&\E[\lambda_{\max}(\frac{\partial}{\partial x}v(x_i))^\top\frac{\partial}{\partial x} v(x_i))]\cdot((\E|\epsilon_i|)^2+O_p(\sqrt\frac{1}{n}))+O_p(\sqrt\frac{m^2}{n}).
\end{align*}


%

In addition, for the class of linear models, we have 
$$
\hat\Delta(\mathcal{A}_{lin})=d\cdot(\E|(\beta_2^\top x)^2+\epsilon|)^2\in[(\|\beta_2\|_2^2\sigma_x^2-\sqrt{\frac{2}{\pi}}\sigma_{\epsilon})^2,(\|\beta_2\|_2^2\sigma_x^2+\sqrt{\frac{2}{\pi}}\sigma_{\epsilon})^2].
$$
Therefore, using the exact same statement as the previous case where $p=q=2$, we get the desired result.

%

\subsection{Proof of Theorem \ref{thm:featurelinear} and Corollary \ref{col:randomeffect}}

Now let us first compute the AIF for linear models. 

Specifically, let us consider the regression setting $(x_i, y_i)\in\mathbb R^{m}\times \mathbb R$ are $i.i.d.$ draws from a joint distribution $P_{x,y}$, for $i=1,2,...,n_{train}$. Note that we don't assume linear relationship, but the linear regression model tries to find the best linear approximation by solving $$
\hat\theta=\arg\min_{\theta} \frac{1}{n_{train}}\sum_{i=1}^nl(\theta, x_i, y_i):=\arg\min_{\theta} \frac{1}{2n_{train}}\sum_{i=1}^{n_{train}}(y_i-\theta^T x_i)^2,
$$
where we use $l(\theta, x_i, y_i)=\frac{1}{2}(y_i-\theta^T x_i)^2$ as the loss function. 

Further, let us define $$
\beta_{\min}^{\mathcal L}=\arg\min_{\theta} \E_{P_{x,y}}[\frac{1}{2}(Y-\theta^T X)^2],
$$
denoting the best population linear approximation to $Y$, which makes $Cov(x_i, y_i-\theta^{*\top}x_i)=0$.
Denote $\eta_{i}^{\mathcal L}=y_i-\beta_{\min}^{\mathcal L\top}x_i$, we then have $\E[\eta_{i}^{\mathcal L} x_i]=0$.

Further, denote $\eta_{i}^{\mathcal L}=y_i-\beta_{\min}^{\mathcal L\top}x_i$, and 
 $$
\hat\beta=\arg\min_{\theta} \frac{1}{2n_{train}}\sum_{i=1}(y_i-\theta^T x_i)^2,
$$
and we have $|a^\top(\hat\beta-\beta_{\min}^{\mathcal L}|=O_p(\sqrt{\frac{\|a\|}{n_{train}}})$.

By definition, for $k\in[m]$, 
\begin{align*}
b_k=&|\frac{\partial}{\partial x_{\cdot,k}}l(\hat\beta,x_i,y_i,\mathcal L)|=|y_i- \hat\beta^\top x_i|\cdot |\hat\beta_k|,
\end{align*}
and therefore, by letting $p=q=2$ in Eqn (5) of Theorem 4.1,
\begin{align*}
\psi^i_k&=\frac{b_k^{q-1}}{(\sum_{k=1}^d b_k^q)^\frac{1}{p}}\text{sgn}(\frac{\partial}{\partial x_{\cdot,k}}l(\hat\beta,x_i,y_i,\mathcal M))=\frac{b_k}{(\sum_{k=1}^d b_k^2)^{1/2}}\text{sgn}((y_i- \hat\beta^\top x_i)\cdot \hat\beta_k)\\
&=\frac{(y_i- \hat\beta^\top x_i)\cdot \hat\beta_k}{|y_i- \theta^\top x_i|\cdot \|\hat\beta\|_2}=\frac{\hat\beta_k}{\|\hat\beta\|_2}\cdot\text{sgn}(y_i-\hat\beta^\top x_i).
\end{align*}

As a result
$$
\phi_i=(\psi^i_1,\psi^i_2,\cdots,\psi^i_m)^T=\text{sgn}(y_i-\hat\beta^\top x_i)\cdot\frac{1}{\|\hat\beta\|}\cdot\hat\beta,
$$
and
 \begin{align*}
\frac{\Phi}{\E_{x\sim\hat P_x}\|x\|_2}=&\frac{1}{n_{train}}\sum_{i=1}^{n_{train}}\nabla_{x,\theta} l(\hat\beta,x_i,y_i,\mathcal M)\phi_i\\
=&\frac{1}{n_{train}}\sum_{i=1}^{n_{train}}[(\hat\beta^\top x_i-y_i)\cdot I_d+\hat\beta x_i^\top]\cdot\text{sgn}(y_i-\hat\beta^\top x_i)\cdot\frac{1}{\|\hat\beta\|}\cdot\hat\beta\\
=&\frac{1}{n_{train}\|\hat\beta\|}\sum_{i=1}^{n_{train}}[(\hat\beta^\top x_i-y_i)\cdot \hat\beta+\hat\beta x_i^\top\hat\beta]\cdot\text{sgn}(y_i-\hat\beta^\top x_i)\\
=&-\frac{1}{n_{train}\|\hat\beta\|}\sum_{i=1}^{n_{train}} (y_i\cdot \hat\beta)\cdot\text{sgn}(y_i-\hat\beta^\top x_i)\\
=&-\frac{\hat\beta}{n_{train}\|\hat\beta\|}\sum_{i=1}^{n_{train}} y_i \cdot\text{sgn}(y_i-\hat\beta^\top x_i)\\
=&-\frac{\hat\beta}{n_{train}\|\hat\beta\|}\sum_{i=1}^{n_{train}} (\eta_{i}^{\mathcal L}+\beta_{\min}^{\mathcal L\top} x_i) \cdot\text{sgn}(y_i-\hat\beta^\top x_i)\\
=&-\frac{\hat\beta}{n_{train}\|\hat\beta\|}\sum_{i=1}^{n_{train}}\beta_{\min}^{\mathcal L\top} x_i \cdot\text{sgn}(y_i-\hat\beta^\top x_i)-\frac{\hat\beta}{n\|\hat\beta\|}\sum_{i=1}^{n_{train}} \eta_{i}^{\mathcal L} \cdot\text{sgn}(y_i-\hat\beta^\top x_i)\\
=&-\frac{\hat\beta}{n_{train}\|\hat\beta\|}\sum_{i=1}^{n_{train}}\beta_{\min}^{\mathcal L\top} x_i \cdot\text{sgn}(\eta_{i}^{\mathcal L})-\frac{\hat\beta}{n_{train}\|\hat\beta\|}\sum_{i=1}^{n_{train}} \eta_{i}^{\mathcal L} \cdot\text{sgn}(\eta_{i}^{\mathcal L})\\
&+\frac{\hat\beta}{n_{train}\|\hat\beta\|}\sum_{i=1}^{n_{train}}\beta_{\min}^{\mathcal L\top} x_i \cdot (\text{sgn}(\eta_{i}^{\mathcal L})-\text{sgn}(\eta_{i}^{\mathcal L}-(\hat\beta-\beta_{\min}^{\mathcal L})^{\top} x_i))\\
&+\frac{\hat\beta}{n_{train}\|\hat\beta\|}\sum_{i=1}^{n_{train}} \eta_{i}^{\mathcal L} \cdot(\text{sgn}(\eta_{i}^{\mathcal L})-\text{sgn}(\eta_{i}^{\mathcal L}-(\hat\beta-\beta_{\min}^{\mathcal L})^\top x_i)).
\end{align*} 

Then we have
$$
\Pb(\text{sgn}(\eta_{i}^{\mathcal L})\neq\text{sgn}(\eta_{i}^{\mathcal L}-(\hat\beta-\beta_{\min}^{\mathcal L})^\top x_i)\le\Pb(|\epsilon|\le |(\hat\beta-\beta_{\min}^{\mathcal L})^\top x_i|)=O(\sqrt\frac{1}{n_{train}})=o(1).
$$

Recall that {\color{black}$\E[x_i\text{sgn}(\eta_{i}^{\mathcal L})]=\E[x_i\text{sgn}((x_i^\top\beta_2^*)^2+\xi_i)]=0$}, we have $$
\frac{1}{n_{train}}\sum_{i=1}^{n_{train}}\beta_{\min}^{\mathcal L\top} x_i \cdot\text{sgn}(\eta_{i}^{\mathcal L})=O_p(\frac{1}{\sqrt{n_{train}}}).
$$

Then, we have $$
\frac{\Phi}{\E_{x\sim\hat P_x}\|x\|_2}=-\frac{\hat\beta}{\|\hat\beta\|}(\frac{1}{n_{train}}\sum_{i=1}^{n_{train}} |\eta_{i}^{\mathcal L}|+O_p(\frac{1}{\sqrt{n_{train}}}))=-\frac{\hat\beta}{\|\hat\beta\|}(\E |\eta_{i}^{\mathcal L}|+O_p(\frac{1}{\sqrt{n_{train}}}))
$$

Moreover, the Hessian matrix 
$$
H_{\theta}(X^e,Y^e)=1/{n_{test}}\sum_{i=1}^{n_{test}}\nabla^2_{\theta} l(\hat\beta,x_i^e,y_i^e;\mathcal{A})=\frac{1}{n_{test}}X^{e\top}X^{e}=\sigma_x^2 I+O_p(\sqrt\frac{m}{n_{test}}).
$$
Then, we have \begin{align}
\hat S_{\epsilon}(\mathcal{L})&=\Phi^\top H^{-1}_{\theta}(X^e,Y^e)\Phi=(\E_{x\sim\hat P_x}\|x\|_2)^2(\E|\eta_{i}^{\mathcal L}|+O_p(\frac{1}{\sqrt{n_{train}}}))^2 \frac{\hat\beta^\top(\sigma_x^2 I+O_p(\sqrt\frac{m}{n_{test}}))^{-1}\hat\beta}{\|\hat\beta\|^2}\nonumber\\
&=\epsilon^2(\E_{x\sim\hat P_x}\|x\|_2)^2\cdot[(\E|\eta_{i}^{\mathcal L}|)^2+O_p(\frac{1}{\sqrt{n_{train}}})]\cdot (\sigma_x^{-2} +O_p(\sqrt\frac{m}{n_{test}}))\nonumber\\
&=\epsilon^2(\E_{x\sim\hat P_x}\|x\|_2)^2\cdot(\E|\eta_{i}^{\mathcal L}|)^2\cdot \sigma_x^{-2} +O_p(\sqrt{\frac{1}{n_{train}}+{\frac{m}{n_{test}}}}).\label{eq:linear}
\end{align}

Now, let us consider the random effect model in Corollary 5.1,  when the true model is $y=\beta^\top x+\xi$, where $x\in\R^M$, $\beta_1, ..., \beta_M\stackrel{i.i.d.}{\sim} N(0,1)$,  $\xi\sim \cN(0,\sigma_\xi^2)$, and $x_1,...,x_n\stackrel{i.i.d.}{\sim} \cN(0,\sigma_x^2 I_M)$. Then when we only include $m$ features in the linear predictive model, the residual $$
\eta_i^{\mathcal L}=\xi+x_{i,m+1}\beta_{m+1}+...+x_{i,M}\beta_{M}.
$$
Then conditional on $\beta$, we have $$
\eta_i^{\mathcal L}\sim N(0,\sigma_\xi^2+(\beta_{m+1}^2+...+\beta_{M}^2)\sigma_x^2).
$$

We then have 
$
\E[|\eta_i^{\mathcal L}|]^2={\frac{2}{\pi}}({\sigma_\xi^2+(\beta_{m+1}^2+...+\beta_{M}^2)\sigma_x^2})
$
Take expectation w.r.t $\beta$, we have $\E[|\eta_i^{\mathcal L}|]^2={\frac{2}{\pi}}(\sigma_\xi^2+(M-m)\sigma_x^2)$.

For $\E_{x\sim\hat P_x}\|x\|_2$, we have $$
\E_{x\sim\hat P_x}\|x\|_2=\E[\sqrt{\beta_1^2+...+\beta_m^2}]=\sqrt\frac{\frac{m+1}{2}}{\frac{m}{2}}.
$$ 

Plug into \eqref{eq:linear}, we get 
$$
\E[\hat{S}_\varepsilon(\cL)]=\frac{4\epsilon^2}{\pi\sigma_x^{2}}\frac{\Gamma^2(\frac{m+1}{2})}{\Gamma^2(\frac{m}{2})}\cdot((M-m)\sigma_x^2+\sigma_{\xi}^2) +O_p(\varepsilon^2\cdot\sqrt{\frac{1}{n_{train}}+{\frac{m}{n_{test}}}}).
$$

\subsection{Proof of Theorem \ref{thm:upperbound}}

Now let us consider the general basis of the regression setting $(x_i, y_i)\in\mathbb R^{m}\times \mathbb R$ are $i.i.d.$ draws from a joint distribution $P_{x,y}$, for $i=1,2,...,n$. Suppose we use the basis $ v( x)=(v_1( x),...,v_d( x))=(x_1,...,x_m,x_1^2/2,...,x_m^2/2, \{x_jx_k\}_{j<k}),$ to approximate $y$, and try to solve
 $$
\hat\theta=\arg\min_{\theta\in\R^m} \frac{1}{n}\sum_{i=1}^nl(\theta, x_i, y_i):=\arg\min_{\theta} \frac{1}{2n}\sum_{i=1}^n(y_i-\theta^T  v(x_i))^2.
$$

Further, let us define $$
\theta^*=\arg\min_{\theta} \E_{P_{x,y}}[\frac{1}{2}(Y-\theta^T  v(X))^2],
$$
denoting the best population linear approximation to $Y$.

Denote $\xi_i=y_i-\theta^{*\top} v(x_i)$, and let
$$
\hat\theta=\arg\min_{\theta} \frac{1}{2n}\sum_{i=1}(y_i-\theta^T  v(x_i))^2.
$$

We have $\|\hat\theta- \theta^*\|=O_p(\sqrt\frac{m}{n})$.

By definition, for $k\in[m]$, 
\begin{align*}
b_k=&|\frac{\partial}{\partial x_{\cdot,k}}l(\htheta,x_i,y_i)|=|y_i- \hat\theta^\top  v(x_i)|\cdot |\hat\theta^\top\frac{\partial}{\partial x_{\cdot,k}} v(x_i)|=|y_i- \hat\theta^\top  v(x_i)|\cdot |\hat\theta^\top\frac{\partial}{\partial x} v(x_i)e_k|.
\end{align*}

Therefore, by letting ${p=q=2}$ in Eqn (5) of Theorem 4.1,
\begin{align*}
\psi^i_k&=\frac{b_k^{q-1}}{(\sum_{k=1}^d b_k^q)^\frac{1}{p}}\text{sgn}(\frac{\partial}{\partial x_{\cdot,k}}l(\htheta,x_i,y_i))=\frac{b_k}{(\sum_{k=1}^d b_k^2)^{1/2}}\text{sgn}(\frac{\partial}{\partial x_{\cdot,k}}l(\htheta,x_i,y_i))\\
&=\frac{(\hat\theta^\top  v(x_i)-y_i)\cdot \hat\theta^\top\frac{\partial}{\partial x} v(x_i)e_k}{|y_i- \hat\theta^\top v( x_i)|\cdot \|\hat\theta^\top\frac{\partial}{\partial x} v(x_i)\|_2}=\frac{\hat\theta^\top\frac{\partial}{\partial x} v(x_i)e_k}{\|\hat\theta^\top\frac{\partial}{\partial x} v(x_i)\|}\cdot\text{sgn}(\hat\theta^\top  v(x_i)-y_i).
\end{align*}

As a result
$$
\phi_i^\top=( \psi^i_1, \psi^i_2,\cdots, \psi^i_m)=\text{sgn}(\hat\theta^\top  v(x_i)-y_i)\cdot\frac{1}{\|\hat\theta^\top\frac{\partial}{\partial x} v(x_i)\|}\cdot\hat\theta^\top\frac{\partial}{\partial x} v(x_i),
$$
and
$$
\nabla_{x} l(\hat\theta,x_i,y_i)=(\hat\theta^\top v(x_i)-y_i)\cdot(\frac{\partial}{\partial x} v(x_i))^\top\hat\theta
$$
$$
\nabla_{x,\theta} l(\hat\theta,x_i,y_i)= v(x_i)\hat\theta^\top\frac{\partial}{\partial x} v(x_i)+(\hat\theta^\top v(x_i)-y_i)\cdot \frac{\partial}{\partial x} v(x_i)
$$

Then
 \begin{align*}
\frac{\Phi}{\E_{x\sim\hat P_x}\|x\|_2}=&\frac{1}{n}\sum_{i=1}^n\nabla_{x,\theta} l(\htheta,x_i,y_i)\phi_i\\
=&\frac{1}{n}\sum_{i=1}^n[(\hat\theta^\top  v(x_i)-y_i)\cdot \frac{\partial}{\partial x} v(x_i)+ v(x_i)\hat\theta^\top\frac{\partial}{\partial x} v(x_i)]\cdot\text{sgn}(\hat\theta^\top  v(x_i)-y_i)\cdot\\&\frac{1}{\|\hat\theta^\top\frac{\partial}{\partial x} v(x_i)\|}\cdot(\frac{\partial}{\partial x} v(x_i))^\top \hat\theta\\
=&\frac{1}{n\|\hat\theta^\top\frac{\partial}{\partial x} v(x_i)\|}\sum_{i=1}^n[(\hat\theta^\top v( x_i)-y_i)\cdot \\
&\frac{\partial}{\partial x} v(x_i)(\frac{\partial}{\partial x} v(x_i))^\top \hat\theta+  v(x_i)\cdot\|\hat\theta^\top\frac{\partial}{\partial x} v(x_i)\|^2 ]\cdot\text{sgn}(\hat\theta^\top  v(x_i)-y_i)\\
=&\frac{1}{n\|\hat\theta^\top\frac{\partial}{\partial x} v(x_i)\|}\sum_{i=1}^n |\hat\theta^\top v( x_i)-y_i|\cdot \frac{\partial}{\partial x} v(x_i)(\frac{\partial}{\partial x} v(x_i))^\top \hat\theta\\
&+\frac{1}{n}\sum_{i=1}^n  v(x_i)\cdot\|\hat\theta^\top\frac{\partial}{\partial x} v(x_i)\|\cdot\text{sgn}(\hat\theta^\top  v(x_i)-y_i).
\end{align*} 

Recall that we assume $\E[\text{sgn}(\epsilon_i) x_i]=0$, then we have $$
\frac{\Phi}{\E_{x\sim\hat P_x}\|x\|_2}=\E[\frac{ |\hat\theta^\top v( x_i)-y_i|}{\|\hat\theta^\top\frac{\partial}{\partial x}\bm v(x_i)\|}\frac{\partial}{\partial x}\bm v(x_i)(\frac{\partial}{\partial x}\bm v(x_i))^\top\hat\theta]+O_p(\sqrt\frac{d}{n}).
$$
Then, since
\begin{align*}
&\|\E[\frac{ |\hat\theta^\top v( x_i)-y_i|}{\|\hat\theta^\top\frac{\partial}{\partial x} v(x_i)\|}\frac{\partial}{\partial x} v(x_i)(\frac{\partial}{\partial x} v(x_i))^\top\hat\theta]\|^2\\
\le&\E[\|\frac{|\hat\theta^\top v( x_i)-y_i|}{\|\hat\theta^\top\frac{\partial}{\partial x} v(x_i)\|}\frac{\partial}{\partial x} v(x_i)(\frac{\partial}{\partial x} v(x_i))^\top\hat\theta\|^2]\\
\le&\E[\|(\frac{\partial}{\partial x} v(x_i))^\top\frac{\partial}{\partial x} v(x_i)\|_2\cdot |\hat\theta^\top v( x_i)-y_i|^2]\\
\le&\E[\|(\frac{\partial}{\partial x} v(x_i))^\top\frac{\partial}{\partial x} v(x_i)\|_2\cdot |\theta^{*\top} v( x_i)-y_i|^2]+O_p(\sqrt\frac{d}{n})\\
\le&\E[\lambda_{\max}(\frac{\partial}{\partial x}\bm v(x_i))^\top\frac{\partial}{\partial x} v(x_i))]\cdot((\E|\xi_i|)^2+O_p(\sqrt\frac{1}{n}))+O_p(\sqrt\frac{d}{n})
\end{align*}

%
%
%
%

Moreover, \begin{align*}
H_{\theta}(X^e,Y^e)&=\frac{1}{n}\sum_{i=1}^n v(x_i) v(x_i)^\top=E[v(x_i) v(x_i)^\top]+O_p(\sqrt\frac{d}{n})
\end{align*}

Then

\begin{align*}
\hat S_{\epsilon}(\mathcal{GL})\le&\frac{1}{\lambda_{\min}(E[v(x_i) v(x_i)^\top])}\E[\lambda_{\max}(\frac{\partial}{\partial x}v(x_i))^\top\frac{\partial}{\partial x} v(x_i))]\cdot((\E|\xi_i|)^2+O_p(\sqrt\frac{1}{n}))+O_p(\sqrt\frac{m^2}{n}).
\end{align*}

\subsection{Proof of Theorem \ref{thm:rs}}

Now let us first recall the AIF for linear models. 

Specifically, let us consider the regression setting $(x_i, y_i)\in\mathbb R^{d}\times \mathbb R$ are $i.i.d.$ draws from a joint distribution $P_{x,y}$, for $i=1,2,...,n$. Note that we don't assume linear relationship, but the linear regression model tries to find the best linear approximation by solving $$
\hat\theta=\arg\min_{\theta} \frac{1}{n}\sum_{i=1}^nl(\theta, \tilde x_i, y_i):=\arg\min_{\theta} \frac{1}{2n}\sum_{i=1}^n(y_i-\theta^T \tilde x_i)^2,
$$
where $\tilde x_i=x_i+\vartheta_i$ we use $l(\theta, \tilde x_i, y_i)=\frac{1}{2}(y_i-\theta^T \tilde x_i)^2$ as the loss function. 

Further, let us define 
$$
\beta^*=\arg\min_{\theta} \E_{P_{x,y}}[\frac{1}{2}(y-\theta^T  x)^2]=(\E[ x x^\top])^{-1}\E[ x y]=(\sigma_x^2)^{-1}\E[xy],
$$
$$
\beta_{\min}^{\mathcal L}=\arg\min_{\theta} \E_{P_{x,y}}[\frac{1}{2}(y-\theta^T \tilde x)^2]=(\E[\tilde x\tilde x^\top])^{-1}\E[\tilde x y]=(\sigma_x^2+\sigma_r^2)^{-1}\E[xy]=\frac{\sigma_x^2}{\sigma_x^2+\sigma_r^2}\beta^*,
$$
denoting the best population linear approximation to $Y$, which makes $Cov(x_i, y_i-\theta^{*\top}x_i)=0$.
Denote $\eta_{i}^{\mathcal L}=y_i-\beta_{\min}^{\mathcal L\top}x_i$, we then have $\E[\eta_{i}^{\mathcal L} x_i]=0$.

Suppose $y_i=\beta^* x_i+\xi_i$, then  $y_i-\beta_{\min}^{\mathcal L\top}\tilde x_i=y_i-\beta_{\min}^{\mathcal L\top} x_i-\beta_{\min}^{\mathcal L\top}\vartheta_i=(\beta^*-\beta_{\min}^{\mathcal L})^\top x_i+\xi_i-\beta_{\min}^{\mathcal L\top}\vartheta_i$

$\E([y_i-\beta_{\min}^{\mathcal L\top}\tilde x_i) \tilde x_i]=0$

\begin{align*}
Var(y_i-\beta_{\min}^{\mathcal L\top}\tilde x_i)=&Var(y_i-\frac{\sigma_x^2}{\sigma_x^2+\sigma_r^2}\beta^{*\top}\tilde x_i)\\
=&Var(y_i-\frac{\sigma_x^2}{\sigma_x^2+\sigma_r^2}\beta^{*\top} x_i)+Var(\frac{\sigma_x^2}{\sigma_x^2+\sigma_r^2}\beta^{*\top}\vartheta_i)\\
=&Var(\beta^{*\top} x_i+\xi_i-\frac{\sigma_x^2}{\sigma_x^2+\sigma_r^2}\beta^{*\top} x_i)+Var(\frac{\sigma_x^2}{\sigma_x^2+\sigma_r^2}\beta^{*\top}\vartheta_i)\\
=&Var(\frac{\sigma_r^2}{\sigma_x^2+\sigma_r^2}\beta^{*\top} x_i)+Var(\xi_i)+Var(\frac{\sigma_x^2}{\sigma_x^2+\sigma_r^2}\beta^{*\top}\vartheta_i)\\
=&(\frac{\sigma_r^2\sigma_x^2}{\sigma_x^2+\sigma_r^2})\|\beta^*\|_2^2+\sigma_{\xi}^2+(\frac{\sigma_r^2\sigma_x^2}{\sigma_x^2+\sigma_r^2})\|\beta^*\|_2^2\\
=&(\frac{2\sigma_r^2\sigma_x^2}{\sigma_x^2+\sigma_r^2})\|\beta^*\|_2^2+\sigma_{\xi}^2
\end{align*}

Further, denote $\eta_{i}^{\mathcal L}=y_i-\beta_{\min}^{\mathcal L\top}x_i$, and 
 $$
\hat\beta=\arg\min_{\theta} \frac{1}{2n}\sum_{i=1}(y_i-\theta^T x_i)^2,
$$
and we have $\|\hat\beta-\beta_{\min}^{\mathcal L}\|_2=O_p(\sqrt{\frac{m}{n}})$.
Then, we have $$
\frac{\Phi}{\E_{x\sim\hat P_x}\|x\|_2}=-\frac{\hat\beta}{\|\hat\beta\|}(\frac{1}{n}\sum_{i=1}^n |\eta_{i}^{\mathcal L}|+O_p(\frac{1}{\sqrt n}))=-\frac{\hat\beta}{\|\hat\beta\|}(\E |\eta_{i}^{\mathcal L}|+O_p(\frac{1}{\sqrt n}))
$$

Moreover, the Hessian matrix on the test data
$$
H_{\theta}(X^e,Y^e)=1/{n'}\sum_{i=1}^{n'}\nabla^2_{\theta} l(\hat\beta,x_i^e,y_i^e;\mathcal{A})=\frac{1}{n}X^{e\top}X^{e}=\sigma_x^2 I+O_p(\sqrt\frac{m}{n}).
$$

Then we have \begin{align*}
\hat S_{\epsilon}(\mathcal{L})&=\Phi^\top H^{-1}_{\theta}(X^e,Y^e)\Phi=(\E_{x\sim\hat P_x}\|x\|_2)^2(\E|\eta_{i}^{\mathcal L}|+O_p(\frac{1}{\sqrt n}))^2 \frac{\hat\beta^\top(\sigma_x^2 I+O_p(\sqrt\frac{m}{n}))^{-1}\hat\beta}{\|\hat\beta\|^2}\\
&=(\E_{x\sim\hat P_x}\|x\|_2)^2\cdot\epsilon^2\cdot[(\E|\eta_{i}^{\mathcal L}|)^2+O_p(\frac{1}{\sqrt n})]\cdot (\sigma_x^{-2} +O_p(\sqrt\frac{m}{n})).
\end{align*}
and
\begin{align*}
\hat S_{\epsilon}(\mathcal{L}_{noise})&=\Phi^\top H^{-1}_{\theta}(X^e,Y^e)\Phi=(\E_{x\sim\hat P_x}\|x\|_2)^2(\E|\eta_{i}^{\mathcal L_{noise}}|+O_p(\frac{1}{\sqrt n}))^2 \frac{\hat\beta^\top((\sigma_x^2+\sigma_r^2) I+O_p(\sqrt\frac{m}{n}))^{-1}\hat\beta}{\|\hat\beta\|^2}\\
&=(\E_{x\sim\hat P_x}\|x\|_2)^2\cdot\epsilon^2\cdot[(\E|\eta_{i}^{\mathcal L_{noise}}|)^2+O_p(\frac{1}{\sqrt n})]\cdot ((\sigma_x^2+\sigma_r^2)^{-1} +O_p(\sqrt\frac{m}{n})).
\end{align*}

Then \begin{align*}
\frac{\hat S_{\epsilon}(\mathcal{L}_{noise})}{\hat S_{\epsilon}(\mathcal{L}_{})}=&\frac{\sigma_x^2}{\sigma_x^2+\sigma_r^2}\cdot\frac{(\frac{2\sigma_r^2\sigma_x^2}{\sigma_x^2+\sigma_r^2})\|\beta_{\min}^{\mathcal L}\|_2^2+\sigma_{\xi}^2}{\sigma_{\xi}^2}+O(\sqrt{\frac{m}{n}})\\
=&\frac{\sigma_x^2/\sigma_{\xi}^2}{\sigma_x^2+\sigma_r^2}\cdot\left({(\frac{2\sigma_r^2\sigma_x^2}{\sigma_x^2+\sigma_r^2})\|\beta_{\min}^{\mathcal L}\|_2^2+\sigma_{\xi}^2}\right)+O(\sqrt{\frac{m}{n}})
\end{align*}
\subsection{Proof of Corollary \ref{col:kernel}}
We consider kernel regression in the following form:
\begin{equation*}
\hat{\cL}_n(\theta,X,Y)=\frac{1}{n}\sum_{i=1}^n\big(y_i-\sum_{j=1}^nK(x_i,x_j)\theta_j\big)^2+\lambda\|\theta\|^2_2.
\end{equation*}
Let us denote $K(x_i):=\big(K(x_i,x_1),K(x_i,x_2),\cdots,K(x_i,x_n)\big)^T$. The proof of Corollary \ref{col:kernel} is almost the same as Theorem \ref{thm:firstorder}, with slightly modification. Actually, the loss can be in a general form as $\hat{\cL}_n(\theta,X,Y)$, our proof for Theorem \ref{thm:firstorder}can still be applied. Since
$$\nabla_{\theta}\hat{\cL}_n(\theta,X,Y)=\frac{1}{n}\sum_{i=1}^n2\Big(K(x_i)^T\theta-y_i\Big)K(x_i)+2\lambda\theta,$$
we have 
\begin{align*}
\nabla_{x_k,\theta}\hat{\cL}_n(\theta,X,Y)&=\frac{2}{n}\sum_{i=1}^n \nabla_{x_k}\Big(K(x_i)^T\theta-y_i\Big)K(x_i)\\
&=\frac{2}{n}\sum_{i=1}^n(  K(x_i)^T\theta\cK_{x_i,x_k}+K(x_i)\theta^T\cK_{x_i,x_k}-y_i\cK_{x_i,x_k}),
\end{align*}
where $\cK_{x_i,x_k}$ is a $n\times m$ matrix in the following form:
$$\cK_{x_i,x_k}=\begin{pmatrix}
 \Big(\frac{\partial K(x_i,x_1)}{\partial x_k}\Big)^T \\ 
 \vdots \\ 
\Big (\frac{\partial K(x_i,x_n)}{\partial x_k}\Big)^T
\end{pmatrix}.$$
Meanwhile,
$$\nabla^2_{\theta\theta}\hat{\cL}_n(\theta,X,Y)=\frac{2}{n}\sum_{i=1}^n K(x_i)K(x_i)^T+2\lambda I.$$
Thus, we have 
\begin{align*}
\ketheta-\ktheta+O(\|\ketheta-\ktheta\|^2_2)=&\big(-\nabla^2_{\theta\theta} \hat{\cL}_n(\ktheta,X,Y)\big)^{-1}\Big(\sum_{i=1}^n\nabla_{x_i,\theta} \hat{\cL}_n(\ktheta,X,Y)\delta_i\\
&+\|\ketheta-\ktheta\|_2\|\delta_i\|_2\Big).
\end{align*}
Besides,
$$\nabla_{x_k}\hat{\cL}_n(\theta,X,Y)=\frac{2}{n}\sum_{i=1}^n\Big(K(x_i)^T\theta-y_i\Big)\cK^T_{x_i,x_k}\theta,$$
By the argument in Theorem \ref{thm:firstorder}, we know
$$\lim_{\varepsilon\rightarrow 0}\frac{\delta_i}{\varepsilon}=\beta_i$$
where
$$\beta_{i,k}=\frac{c_k^{q-1}}{(\sum_{k=1}^m c_k^q)^\frac{1}{p}}\text{sgn}\Big(\nabla_{x_i}\hat{\cL}_n(\ktheta,k)\Big),$$
with $c_k=|\nabla_{x_i}\hat{\cL}_n(\ktheta,k)|$ and $\nabla_{x_i}\hat{\cL}_n(\ktheta,k)$ is short for the $k$-th coordinate of $\nabla_{x_i}\hat{\cL}_n(\theta,X,Y)$.

\subsection{Proof of Theorem \ref{thm:DAIF}}
%
We still have 
\begin{align*}
\hdetheta-\htheta \approx &\big(-\frac{1}{n}\sum_{i=1}^n\nabla^2_\theta L(\htheta,x_i,y_i)\big)^{-1}\Big(\frac{1}{n}\sum_{i=1}^n\nabla_{x,\theta} L(\htheta,x_i,y_i)\delta_i\Big).
\end{align*}

Notice, we can put all the mass in \ref{lemma:approx} on one of $\delta_i$, so, we can put all on the $\delta_i$ with largest $\|\nabla_{x} L(\htheta,x_i,y_i)\|_q$ in order to achieve the maximum. The rest follows directly from the proof of Theorem \ref{thm:firstorder}.


\end{document}